\documentclass[a4paper]{article}
\usepackage[T1]{fontenc}
\usepackage[utf8]{inputenc}
\usepackage{authblk}
\usepackage{amssymb}
\usepackage{amsmath}
\usepackage{amsthm}
\usepackage{bm}
\usepackage{enumerate}
\usepackage{multirow}
\usepackage{wrapfig}
\usepackage{times}
\usepackage{graphicx} % more modern

\usepackage{xcolor}

\usepackage[sort&compress,numbers,square]{natbib}
\usepackage[titlenumbered,ruled,lined,linesnumbered,algo2e]{algorithm2e}

\usepackage{url}
\usepackage{caption}
\usepackage[footnotesize,rm,RM]{subfigure}
\usepackage[scale={0.72,0.743}, hmarginratio=1:1, vmarginratio=1:1, headheight=2ex, headsep = 2ex, footskip = 3.6ex, nomarginpar]{geometry}
\usepackage{color}

\newenvironment{derivation}{\begin{proof}[\textbf{Dualization}]}{\end{proof}}
\newtheorem{theorem}{Theorem}
\newtheorem{problem}[theorem]{Problem}
\newtheorem{lemma}[theorem]{Lemma}
\newtheorem{proposition}[theorem]{Proposition}
\newtheorem{definition}[theorem]{Definition}

\newcommand{\nbb}{\ensuremath{\mathbb{N}}}
\newcommand{\np}{\text{AE}}
\newcommand{\hcal}{\ensuremath{\mathcal{H}}}

\newcommand{\gcal}{\mathcal{G}}
\newcommand{\ebb}{\mathbb{E}}
\newcommand{\rbb}{\mathbb{R}}
\newcommand{\ycal}{\mathcal{Y}}
\newcommand{\vbeta}{\boldsymbol{\beta}}
\newcommand{\valpha}{\boldsymbol{\alpha}}

\title{Localized Multiple Kernel Learning---A Convex Approach}
\author[1]{Yunwen Lei\thanks{yunwelei@cityu.edu.hk}}
\author[2]{Alexander Binder\thanks{alexander\_binder@sutd.edu.sg}}
\author[3]{\"Ur\"un Dogan\thanks{udogan@microsoft.com}}
\author[4]{Marius Kloft\thanks{kloft@hu-berlin.de}}
\affil[1]{Department of Mathematics, City University of Hong Kong}
\affil[2]{Information Systems Technology and Design Pillar (ISTD), Singapore University of Technology and Design}
\affil[3]{Microsoft Research, Cambridge CB1 2FB, UK}
\affil[4]{Department of Computer Science, Humboldt University of Berlin}
\date{}

\begin{document}
\maketitle
\begin{abstract}
%%%%%%%%%%%%%%%%%%%%%%%%%%%%%%%%%%%%%%%%%%%%%%%%%%%%%%%%%%%%%%%%%%
We propose a localized approach to multiple kernel learning that can be
formulated as a \emph{convex} optimization problem over a given cluster structure.
For which we obtain generalization error guarantees and
derive an optimization algorithm based on the Fenchel dual representation.
Experiments on real-world datasets from the application domains of computational biology and computer vision
show that convex localized multiple kernel learning
can achieve higher prediction accuracies than its global and non-convex local counterparts.

\medskip
\textbf{Keywords}: Multiple kernel learning, Localized algorithms, Generalization analysis
\end{abstract}%The theoretical analysis shows that the choice of cluster membership assignment crucially controls the generalization error.

%%%%%%%%%%%%%%%%%%%%%%%%%%%%%%%%%%%%%%%%%%%%%%%%%%%%%%%%%%%%%%%%%%
\section{Introduction}
%%%%%%%%%%%%%%%%%%%%%%%%%%%%%%%%%%%%%%%%%%%%%%%%%%%%%%%%%%%%%%%%%%

Kernel-based methods %\citep[e.g.,][]{SchSmo02}
such as support vector machines %\cite{BosGuyVap92,CorVap95}
have found diverse applications
due to their distinct merits such as the descent computational complexity, high usability, and the solid mathematical foundation
\citep[e.g.,][]{SchSmo02}.
%\citep[e.g.,][]{mohri2012foundations}.
%The learning and data representation processes are decoupled in a modular fashion, which allows us to obtain complex learning machines from simple linear ones in a canonical way.
The performance of such algorithms, however, crucially depends on the involved kernel function
as it intrinsically specifies the feature space where the learning process is implemented, and thus provides a similarity measure on the input space.
Yet in the standard setting of these methods the choice of the involved kernel is typically left to the user.

A substantial step toward the complete automatization of kernel-based machine learning is
achieved in \citet{lanckriet2004learning}, who introduce the \emph{multiple kernel learning} (MKL)  %or \emph{learning kernels}
framework \citep{Goenen11}.
MKL offers a principal way of encoding complementary information with distinct base kernels and automatically learning an optimal combination of those \citep{sonnenburg2006large}. %\cite{sonnenburg2006large,GehlerNowozin2008}.
MKL can be phrased as a single convex optimization problem, which
facilitates the application of efficient numerical optimization strategies \citep{bach2004multiple,kloft2009efficient,sonnenburg2006large,rakotomamonjy2008simplemkl,xu2010simple,kloft2008automatic,yang2011efficient}
and theoretical understanding of the generalization performance of the resulting models \citep{srebro2006learning,cortes2010generalization,kloft2010unifying,KloBla11,kloft2012convergence,cortes2013learning,ying2009generalization,hussain2011improved,lei2014refined}.
While early sparsity-inducing approaches failed to live up to its expectations in terms of improvement over uniform combinations of kernels \citep[cf.][and references therein]{Cor09},
it was shown that improved predictive accuracy can be achieved by employing appropriate regularization \citep{kloft2011lp,kloft2008non}.

%This is particularly significant in real-world applications such as bioinformatics and computer vision, where data either frequently arises from multiple heterogeneous sources, describing different properties of one and the same object, or is represented by various complementary views \cite{BenOngSonSchRae08,GehlerNowozin2008,kloft2011lpb}.

Currently, most of the existing algorithms fall into the \emph{global} setting of MKL,
in the sense that all input instances share the same kernel weights.
However, this ignores the fact that instances may require sample-adaptive kernel weights.

For instance, consider the two images of a horses given to the right.
Multiple kernels can be defined, capturing the shapes in the image and the color distribution over various channels.
On the image to the left,
\begin{wrapfigure}{r}{0.5\textwidth}
\vspace{-0.25cm}
\includegraphics[width=0.252\textwidth]{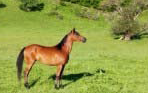}
\includegraphics[width=0.24\textwidth]{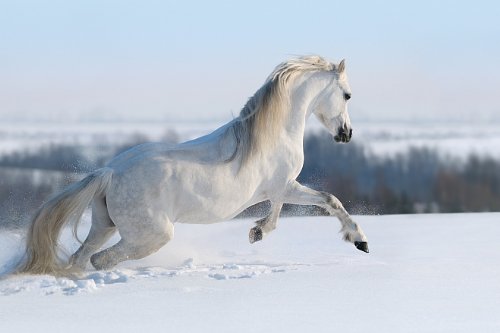}
\vspace{-.6cm}
\end{wrapfigure}
the depicted horse and the image backgrounds exhibit distinctly different color distributions,
while for the image to the right the contrary is the case.
Hence, a color kernel is more significant to detect a horse in the image to the left than for the image the right.
This example motivates studying localized approaches to MKL \citep{yang2009group,gonen2008localized,limultiple,lei2015theory,mu2011non,han2012probability}. %dong2013subcategory
%Which is also backed up by recent studies in computer vision that have demonstrated that such a local learning strategy
%could be superior to its global counterparts, especially for data sets that are not evenly distributed,
%that is, data sets in which examples have different spatial resolutions \cite{dong2013subcategory,yang2009group,gonen2008localized}.
%Similar, in gene finding, DNA can be unevenly distributed over the genome, with varying motif lengths, requiring
%different degrees of the involved string kernels \citep{BenOngSonSchRae08}.

Existing approaches to localized MKL (reviewed in Section \ref{sec:work}) optimize  \emph{non-convex} objective functions.
This puts their generalization ability into doubt.
Indeed, besides the recent work by \citep{lei2015theory}, the generalization performance of localized MKL algorithms
(as measured through large-deviation bounds) is poorly understood, which potentially could make these algorithms prone to overfitting.
Further potential disadvantages of non-convex localized MKL approaches include computationally difficulty in finding good local minima and
the induced lack of reproducibility of results (due to varying local optima).

This paper presents a \emph{convex} formulation of localized multiple kernel learning, which is formulated
as a single convex optimization problem over a precomputed cluster structure, obtained through a potentially
convex or non-convex clustering method.
%1. The training instances are clustered. 2. The kernel weights are computed for each cluster through a
%single convex optimization problem for which
We derive an efficient optimization algorithm based on Fenchel duality.
Using Rademacher complexity theory, we establish large-deviation inequalities for localized MKL,
showing that the smoothness in the cluster membership assignments crucially controls the generalization error.
Computational experiments on data from the domains of computational biology and computer vision show that
the proposed convex approach can achieve higher prediction accuracies than its global and non-convex local counterparts
(up to $+5\%$ accuracy for splice site detection).
%For image categorization and protein folding prediction, we achieve empirical results that both surpass the previously known best results on these data sets.

%The remainder of this paper is structured as follows. We start by discussing related work in Section \ref{sec:work}.
%In Section \ref{sec:methodology} the convex model of localized multiple kernel learning is introduced.
%Section \ref{sec:dual} gives both the complete and partial dual problems, based on which
%we present an efficient optimization algorithm in Section \ref{sec:algorithms}.
%We report on theoretical results including generalization error bounds in Section \ref{sec:bounds}.
%Empirical results for the application domains of visual image recognition and protein fold class prediction
%are presented in Section \ref{sec:empirical}. Section \ref{sec:conclusions} concludes the paper.

%%%%%%%%%%%%%%%%%%%%%%%%%%%%%%%%%%%%%%%%%%%%%%%%%%%%%%%%%%%%%%%%%%
\vspace{-0.1cm}
\subsection{Related Work\label{sec:work}}
\vspace{-0.1cm}
%%%%%%%%%%%%%%%%%%%%%%%%%%%%%%%%%%%%%%%%%%%%%%%%%%%%%%%%%%%%%%%%%%

\citet{gonen2008localized} initiate the work on localized MKL by introducing gating models
$$\vspace{-0.03cm}f(x)=\sum_{m=1}^M\eta_m(x;v)\langle w_m,\phi_m(x)\rangle+b, \quad \eta_m(x;v)\propto\exp(\langle v_m,x\rangle+v_{m0})\vspace{-0.03cm}$$
to achieve local assignments of kernel weights, resulting in a non-convex MKL problem.
%The optimization of gating models and classifiers are formulated in a joint manner, and solved with an alternate strategy.
%However, the optimization problem in \cite{gonen2008localized} is non-convex due to the non-linearity introduced by the gating model.
%Therefore, this optimization algorithm may be trapped into a suboptimal solution, whose quality is sensitive to the initial value of the parameters.
%where $M$ is the number of kernels,
%$\eta_m(x;v)$ is a parametric gating model assigning a weight to $\phi_m(x)$ as a function of $x$, and $v$ encodes the parameters of the gating model.
%The gating function is used to divide the input space into different regions, and then local weights are assigned to kernels in a data-dependent way.
%The optimization of the gating model and the kernel-based classifier is carried out in a joint manner,
%and solved by alternating the optimization with respect to the gating function and the classifier coefficients, respectively.
%Other than binary classification problems, \citet{gonen2010localized,gonen2013localized} also extended the localized MKL to solve other learning problems, including regression, novelty detection and multi-class classification.
%Furthermore their approach assigns a distinct kernel weight to each example.
%For instance in object categorization, however, heavily respecting individual examples may overwhelm the intrinsic properties of a category, making the classifier less reliable~\citep{yang2009group}.
To not overly respect individual samples, \citet{yang2009group} give a group-sensitive formulation of localized MKL, where kernel weights vary at, instead of the example level, the group level.
\citet{mu2011non} also introduce a non-uniform MKL allowing the kernel weights to vary at the cluster-level and tune the kernel weights under the graph embedding framework.
%\cite{song2011localized} construct the involved local models in an independent fashion, ignoring coupling among different localities.
%As in \cite{gonen2008localized}, the optimization problem is solved by a gradient descent algorithm wrapping around a canonical SVM solver.
%Their approach is, however, non-convex, and the proposed algorithm is not guaranteed to convergence to a global optimum.
%That is, kernel weights vary at, instead of the example level, the group level.
\citet{han2012probability} built on \citet{gonen2008localized} by complementing the spatial-similarity-based kernels with probability confidence kernels reflecting the likelihood of examples belonging to the same class. \citet{limultiple} propose a multiple kernel clustering method by maximizing \emph{local} kernel alignments.
%As in \cite{gonen2008localized}, the involved optimization problem is non-convex, leading to poor computational efficiency and suboptimal performance.%The parameters of probability confidence kernels and the classifier coefficients are then optimized in an alternate manner~\citep{han2012probability}, by a gradient-descent method and an off-the-shelf SVM solver, respectively. with probability confidence kernels that reflect the degree of confidence to which the involved examples belong to the same class.
\citet{LiuWZY14} present sample-adaptive approaches to localized MKL, where kernels can be switched on/off at the example level by introducing a latent binary vector for each individual sample, which and the kernel weights are then jointly optimized via margin maximization principle.
\citet{moeller2016unified} present a unified viewpoint of localized MKL by interpreting gating functions in terms of local reproducing kernel Hilbert spaces acting on the data.
%, and may produce a suboptimal classifier when these localities are even moderately correlated.
%For this purpose, they firstly used the expectation-maximization algorithm to develop a locality gating model, which is then used to partition the input space of multiple feature representations into a set of localities of simpler data structure. However, at each locality the localized multi-kernel classifier is constructed in an independent fashion. Therefore, it ignores the coupling among different localities, and may produce a suboptimal classifier when these localities are even moderately correlated.%At each locality, a local multi-kernel classifier is independently constructed by using SimpleMKL in \citep{rakotomamonjy2008simplemkl}. The ultimate model is a mixture of different local classifiers, weighted by the gating function. However, the partition of the input space into different localities is based on the assumption that the sample is generated by a mixture of different Gaussian distributions, which may be suboptimal when this assumption does not hold. Finally
All the aforementioned approaches to localized MKL are formulated in terms of \emph{non-convex} optimization problems,
and deep theoretical foundations in the form of generalization error or excess risk bounds are unknown.
Although \citet{cortes2013learning} present a convex approach to MKL based on controlling the local Rademacher complexity,
the meaning of \emph{locality} is different in \citet{cortes2013learning}: it refers to the localization of the hypothesis class,
which can result in sharper excess risk bounds \citep{KloBla11,kloft2012convergence},
and is not related to localized multiple kernel learning. \citet{LiuWYDZ15} extend the idea of sample-adaptive MKL to address the issue with missing kernel information on some examples. More recently, \citet{lei2015theory} propose a MKL method by decoupling the locality structure learning with a hard clustering strategy from optimizing the parameters in the spirit of multi-task learning. They also develop the first generalization error bounds for localized MKL.

%%%%%%%%%%%%%%%%%%%%%%%%%%%%%%%%%%%%%%%%%%%%%%%%%%%%%%%%%%%%%%%%%%
\vspace{-0.1cm}
\section{Convex Localized Multiple Kernel Learning}\label{sec:methodology}
\vspace{-0.1cm}
%%%%%%%%%%%%%%%%%%%%%%%%%%%%%%%%%%%%%%%%%%%%%%%%%%%%%%%%%%%%%%%%%%

%For simplicity, we present our approach for binary classification,
%although the approach is general and can be extended to regression, multi-class classification, and structured output prediction.

%In this section we introduce the proposed learning approach.

\vspace{-0.1cm}
\subsection{Problem setting and notation}
\vspace{-0.1cm}

Suppose that we are given $n$ training samples $(x_1,y_1),\ldots,(x_n,y_n)$ that are partitioned into $l$ disjoint clusters $S_1,\ldots,S_l$ in a probabilistic manner,
meaning that, for each cluster $S_j$, we have a function $c_j:\mathcal{X}\to[0,1]$ indicating the likelihood of $x$ falling into cluster $j$,
i.e., $\sum_{j\in\nbb_l}c_j(x)=1$ for all $x\in\mathcal X$.
Here, for any $d\in\nbb$, we introduce the notation $\nbb_d=\{1,\ldots,d\}$.
Suppose that we are given $M$ base kernels $k_1,\ldots,k_M$ with $k_m(x,\tilde{x})=\langle \phi_m(x),\phi_m(\tilde{x})\rangle_{k_m}$,
corresponding to linear models $f_j(x)=\langle w_j,\phi(x)\rangle+b=\sum_{m\in\nbb_M}\langle w_j^{(m)},\phi_m(x)\rangle+b,$
where $w_{j}=(w_j^{(1)},\ldots,w_j^{(M)})$ and $\phi=(\phi_1,\ldots,\phi_M)$.
%For clarity, we frequently use the notation $\langle\cdot,\cdot\rangle:=\langle\cdot,\cdot\rangle_{k_m}$ and $\|\cdot\|_2:=\|\cdot\|_{k_m}:=\langle\cdot,\cdot\rangle_{k_m}$.
Then we consider the following proposed model, which is a weighted combination of these $l$ local models:
\begin{equation}\label{decision-function}
  f(x)=\sum_{j\in\nbb_l}c_j(x)f_j(x)=\sum_{j\in\nbb_l}c_j(x)\big[\sum_{m\in\nbb_M}\langle w_j^{(m)},\phi_m(x)\rangle\big]+b.
\end{equation}
%We employ a convex loss function $\ell(t,y):\mathbb{R}\times\mathcal{Y}\to\mathbb{R}$ measuring the error incurred from using $t$ to predict $y$ (the convexity assumption is only made w.r.t. the first argument).

\vspace{-0.1cm}
\subsection{Proposed convex localized MKL method}
\vspace{-0.1cm}

Using the above notation, the proposed convex localized MKL model can be formulated as follows.

\begin{problem}[\textsc{Convex Localized Multiple Kernel Learning (CLMKL)---Primal}]\label{prob:primal}
Let $C>0$ and $p\geq1$. Given a loss function $\ell(t,y):\mathbb{R}\times\mathcal{Y}\to\mathbb{R}$ convex w.r.t. the first argument and cluster likelihood functions $c_j:\mathcal{X}\to[0,1]$, $j\in\nbb_l$, solve

\begin{equation}\label{soft-primal}
\tag{P}
\begin{split}
  \inf_{w,t,\beta,b}&\sum_{j\in\nbb_l}\sum_{m\in\nbb_M}\frac{\|w_j^{(m)}\|_2^2}{2\beta_{jm}}+ C\sum_{i\in\nbb_n}\ell(t_i,y_i)\\
  \mbox{s.t.}\;\;
  &\beta_{jm}\geq0,\;\sum_{m\in\nbb_M}\beta^p_{jm}\leq1\quad\forall j\in\nbb_l,m\in\nbb_M\\
  &\sum_{j\in\nbb_l}c_j(x_i)\big[\sum_{m\in\nbb_M}\langle w_j^{(m)},\phi_m(x_i)\rangle\big]+b=t_i,\;\forall i\in\nbb_n.
\end{split}
\end{equation}
\end{problem}

The core idea of the above problem is to use cluster likelihood functions for each example and separate $\ell_p$-norm constraint on the kernel weights $\beta_j:=(\beta_{j1},\ldots,\beta_{jM})$
for each cluster $j$ \citep{kloft2011lp} . Thus each instance can obtain separate kernel weights.
The above problem is convex, since a quadratic over a linear function is convex \citep[e.g.,][p.g. 89]{boyd2004convex}. %, so all occurring summands in formulation \eqref{soft-primal} are convex.
Note that Slater's condition can be directly checked, and thus strong duality holds.
%To our best knowledge, Problem \ref{prob:primal} is the first convex formulation of localized MKL. % considering all training examples in a joint manner.

%However, unlike training a local model independently at each locality, these $l$ local models are optimized jointly in our formulation,
%exploiting that examples in localities close by may convey complementary information to the learning task.
%The regularizer as defined in \eqref{Mahalanobis-regularizer} encodes the relationship among different clusters and imposes a soft constraint
%on how these local models shall be correlated, a graph-regularizer as pioneered in graph-regularized multi-task learning \cite{evgeniou2005learning}.
%The larger the value of $\Sigma_{j\tilde{j}}^{-1}$, the stronger correlations are promoted between the $j$-th and $\tilde{j}$-th local model.
%Note that, if $\Sigma^{-1}$ is the graph Laplacian of an adjacency matrix $W$
%(i.e., $\Sigma^{-1}=D-W$ with $D_{j\tilde{j}}=\delta_{j\tilde{j}}\sum_{k\in\nbb_l}W_{jk}$),
%the regularizer \eqref{Mahalanobis-regularizer} coincides with the graph regularizer employed also in \citet{evgeniou2005learning}:

%%%%%%%%%%%%%%%%%%%%%%%%%%%%%%%%%%%%%%%%%%%%%%%%%%%%%%%%%%%%%%%%%%
\vspace{-0.1cm}
\subsection{Dualization}\label{sec:dual}
\vspace{-0.1cm}
%%%%%%%%%%%%%%%%%%%%%%%%%%%%%%%%%%%%%%%%%%%%%%%%%%%%%%%%%%%%%%%%%%

In this section we derive a dual representation of Problem \ref{prob:primal}.
We consider two levels of duality: a partially dualized problem, with fixed kernel weights $\beta_{jm}$, and the entirely dualized problem with respect to all occurring primal variables. From the former we derive an efficient two-step optimization scheme (Section ~\ref{sec:algorithms}). The latter allows us to compute the duality gap and thus to obtain a sound stopping condition for the proposed algorithm.
We focus on the entirely dualized problem here. The partial dualization is deferred to Supplemental Material \ref{supp:partial_dual}.

%-------------------------------------
\vspace{-0.1cm}
\paragraph{Dual CLMKL Optimization Problem}
%-------------------------------------

For $w_j=(w_j^{(1)},\ldots,w_j^{(M)})$, we define the $\ell_{2,p}$-norm by
$\|w_j\|_{2,p}:=\|(\|w_j^{(1)}\|_{k_1},\ldots,\|w_j^{(M)}\|_{k_M})\|_p=(\sum_{m\in\nbb_M}\|w_j^{(m)}\|^p_{k_m})^{\frac{1}{p}}.$
For a function $h$, we denote by $h^*(x)=\sup_\mu[x^\top\mu-h(\mu)]$ its Fenchel-Legendre conjugate.
This results in the following dual. %We use the notation $\|\cdot\|_2=\|\cdot\|_{k_m}$ for clarity.

\begin{problem}[\textsc{CLMKL---Dual}]\label{thm:soft-complete-dual}
The dual problem of \eqref{soft-primal} is given by
  \begin{align}\label{soft-complete-dual}
	  \tag{D}
    \sup_{\sum_{i\in\nbb_n}\alpha_i=0}\bigg\{-C\sum_{i\in\nbb_n}\ell^*(-\frac{\alpha_i}{C},y_i)-\frac{1}{2}\sum_{j\in\nbb_l}\Big\|\big(\sum_{i\in\nbb_n}\alpha_ic_j(x_i)\phi_m(x_i)\big)_{m=1}^M\Big\|^2_{2,\frac{2p}{p-1}}\bigg\}.
  \end{align}
\end{problem}

\vspace{-0.1cm}
\begin{derivation}
Using Lemma \ref{lem:micchelli} from Supplemental Material \ref{app:proofs_dual} to express the optimal $\beta_{jm}$ in terms of $w_j^{(m)}$, the problem \eqref{soft-primal} is equivalent to
\begin{equation}\label{soft-primal-complete}
\begin{split}
  \inf_{w,t,b}\;&\frac{1}{2}\sum_{j\in\nbb_l}\Big(\sum_{m\in\nbb_M}\|w_j^{(m)}\|_2^{\frac{2p}{p+1}}\Big)^{\frac{p+1}{p}}+C\sum_{i\in\nbb_n}\ell(t_i,y_i)\\
  \mbox{s.t.}\;\;&\sum_{j\in\nbb_l}\big[c_j(x_i)\sum_{m\in\nbb_M}\langle w_j^{(m)},\phi_m(x_i)\rangle\big]+b=t_i,\;\forall i\in\nbb_n.
\end{split}
\end{equation}
Introducing Lagrangian multipliers $\alpha_i,i\in\nbb_n$, the Lagrangian saddle problem of Eq. \eqref{soft-primal-complete} is
\begin{align}\label{soft-complete-dual-derivation}
  &\sup_{\alpha}\inf_{w,t,b}\frac{1}{2}\sum_{j\in\nbb_l}\Big(\sum_{m\in\nbb_M}\|w_j^{(m)}\|_2^{\frac{2p}{p+1}}\Big)^{\frac{p+1}{p}}+C\sum_{i\in\nbb_n}\ell(t_i,y_i)
  -\sum_{i\in\nbb_n}\alpha_i\Big(\sum_{j\in\nbb_l}c_j(x_i)\sum_{m\in\nbb_M}\langle w_j^{(m)},\phi_m(x_i)\rangle+b-t_i\Big)\nonumber\\
  &=\sup_{\alpha}\bigg\{-C\sum_{i\in\nbb_n}\sup_{t_i}[-\ell(t_i,y_i)-\frac{1}{C}\alpha_it_i]-\sup_b\sum_{i\in\nbb_n}\alpha_ib-\nonumber\\
  &\qquad\sup_w\Big[\sum_{j\in\nbb_l}\sum_{m\in\nbb_M}\big\langle w_j^{(m)},\sum_{i\in\nbb_n}\alpha_ic_j(x_i)\phi_m(x_i)\big\rangle-\frac{1}{2}\sum_{j\in\nbb_l}\Big(\sum_{m\in\nbb_M}\|w_j^{(m)}\|_2^{\frac{2p}{p+1}}\Big)^{\frac{p+1}{p}}\Big]\bigg\}\\
  %&\stackrel{\text{def}}{=}\sup_{\sum_{i\in\nbb_n}\alpha_i=0}\bigg\{-C\sum_{i\in\nbb_n}\ell^*(-\frac{\alpha_i}{C},y_i)-
  %\sum_{j\in\nbb_l}\sup_{w_j}\Big[\Big\langle(w_j^{(m)})_{m=1}^M,\big(\sum_{i\in\nbb_n}\alpha_ic_j(x_i)\phi_m(x_i)\big)_{m=1}^M\Big\rangle-\frac{1}{2}\Big\|(w_j^{(m)})_{m=1}^M\Big\|^2_{2,\frac{2p}{p+1}}\Big]\bigg\}\\
  &\stackrel{\text{def}}{=}\sup_{\sum_{i\in\nbb_n}\alpha_i=0}\bigg\{-C\sum_{i\in\nbb_n}\ell^*(-\frac{\alpha_i}{C},y_i)-\sum_{j\in\nbb_l}\Big[\frac{1}{2}\big\|\big(\sum_{i\in\nbb_n}\alpha_ic_j(x_i)\phi_m(x_i)\big)_{m=1}^M\big\|^2_{2,\frac{2p}{p+1}}\Big]^*\bigg\}\nonumber
	%\\
  %&=\sup_{\sum_{i\in\nbb_n}\alpha_i=0}\bigg\{-C\sum_{i\in\nbb_n}\ell^*(-\frac{\alpha_i}{C},y_i)-\frac{1}{2}\sum_{j\in\nbb_l}\Big\|\big(\sum_{i\in\nbb_n}\alpha_ic_j(x_i)\phi_m(x_i)\big)_{m=1}^M\Big\|^2_{2,\frac{2p}{p-1}}\bigg\}.
\end{align}
The result \eqref{thm:soft-complete-dual} now follows by recalling that for a norm $\|\cdot\|$, its dual norm $\|\cdot\|_*$ is  defined by $\|x\|_*=\sup_{\|\mu\|=1}\langle x,\mu\rangle$ and satisfies:
$ (\frac{1}{2}\|\cdot\|^2)^*=\frac{1}{2}\|\cdot\|_*^2$ \citep{boyd2004convex}.
Furthermore, it is straightforward to show that $\|\cdot\|_{2,\frac{2p}{p-1}}$ is the dual norm of $\|\cdot\|_{2,\frac{2p}{p+1}}$.
\end{derivation}

\vspace{-0.1cm}
\subsection{Representer Theorem}
\vspace{-0.1cm}
We can use the above derivation to obtain a lower bound on the optimal value of the primal optimization problem \eqref{soft-primal}, from which we can compute the duality gap using the
theorem below. The proof is given in Supplemental Material \ref{supp_proof_repr}.

\begin{theorem}[\textsc{Representer Theorem}]\label{thm:repre}
For any dual variable $(\alpha_i)_{i=1}^n$ in \eqref{soft-complete-dual}, the optimal primal variable $\{w_j^{(m)}(\alpha)\}_{j,m=1}^{l,M}$ in the Lagrangian saddle problem \eqref{soft-complete-dual-derivation} can be represented as
  $$w_j^{(m)}(\alpha)\!=\!\big[\sum_{\tilde{m}\in\nbb_M}\|\!\sum_{i\in\nbb_n}\!\alpha_ic_j(x_i)\phi_{\tilde{m}}(x_i)\|_2^{\frac{2p}{p-1}}\big]^{-\frac{1}{p}}\big\|\!\sum_{i\in\nbb_n}\!\alpha_ic_j(x_i)\phi_m(x_i)\big\|_2^{\frac{2}{p-1}}\big[\!\sum_{i\in\nbb_n}\alpha_ic_j(x_i)\phi_m(x_i)\big].$$
\end{theorem}

%-------------------------------------
\subsection{Support-Vector Classification}
%-------------------------------------

For the hinge loss, the Fenchel-Legendre conjugate becomes $\ell^*(t,y)=\frac{t}{y}$ (a function of $t$) if $-1\leq\frac{t}{y}\leq0$ and $\infty$ elsewise. Hence, for each $i$, the term $\ell^*(-\frac{\alpha_i}{C},y_i)$ translates to $-\frac{\alpha_i}{Cy_i}$, provided that $0\leq\frac{\alpha_i}{y_i}\leq C$. With a variable substitution of the form $\alpha_i^{\text{new}}=\frac{\alpha_i}{y_i}$, the complete dual problem \eqref{soft-complete-dual} reduces as follows.
\begin{problem}[\sc CLMKL---SVM Formulation]
For the hinge loss, the dual CLMKL problem \eqref{soft-complete-dual} is given by:
\begin{equation}\label{soft-complete-dual-hinge}
\begin{split}
  \sup_{\alpha:0\leq\alpha\leq C,\sum_{i\in\nbb_n}\alpha_iy_i=0}\;&-\frac{1}{2}\sum_{j\in\nbb_l}\Big\|\big(\sum_{i\in\nbb_n}\alpha_iy_ic_j(x_i)\phi_m(x_i)\big)_{m=1}^M\Big\|^2_{2,\frac{2p}{p-1}}+\sum_{i\in\nbb_n}\alpha_i,\\
\end{split}
\end{equation}
\end{problem}

A corresponding formulation for support-vector regression is given in Supplemental Material \ref{sec:regression}.

%%%%%%%%%%%%%%%%%%%%%%%%%%%%%%%%%%%%%%%%%%%%%%%%%%%%%%%%%%%%%%%%%%
\vspace{-0.1cm}
\section{Optimization Algorithms\label{sec:algorithms}}
\vspace{-0.1cm}
%%%%%%%%%%%%%%%%%%%%%%%%%%%%%%%%%%%%%%%%%%%%%%%%%%%%%%%%%%%%%%%%%%
As pioneered in \citet{sonnenburg2006large}, we consider here a two-layer optimization procedure to solve the problem \eqref{soft-primal}
where the variables are divided into two groups: the group of kernel weights $\{\beta_{jm}\}_{j,m=1}^{l,M}$ and the group of weight vectors $\{w_j^{(m)}\}_{j,m=1}^{l,M}$.
In each iteration, we alternatingly optimize one group of variables while fixing the other group of variables.
These iterations are repeated until some optimality conditions are satisfied. To this aim, we need to find efficient strategies to solve the two subproblems.

It is not difficult to show (cf. Supplemental Material \ref{supp:partial_dual}) that, given fixed kernel weights  $\vbeta=(\beta_{jm})$, the CLMKL dual problem is given by
\begin{equation}\label{soft-partial-dual}
  \sup_{\valpha:\sum_{i\in\nbb_n}\alpha_i=0}-\frac{1}{2}\sum_{j\in\nbb_l}\sum_{m\in\nbb_M}\beta_{jm}\Big\|\sum_{i\in\nbb_n}\alpha_ic_j(x_i)\phi_m(x_i)\Big\|_2^2-C\sum_{i\in\nbb_n}\ell^*(-\frac{\alpha_i}{C},y_i),
\end{equation}
which is a standard SVM problem using the kernel
\begin{equation}\label{kernel-general-loss}
  \tilde{k}(x_i,x_{\tilde{i}}):=\sum_{m\in\nbb_M}\sum_{j\in\nbb_l}\beta_{jm}c_j(x_i)c_j(x_{\tilde{i}})k_m(x_i,x_{\tilde{i}})
\end{equation}
This allows us to employ very efficient existing SVM solvers \citep{chang2011libsvm}. % such as LIBSVM or SVM$^\text{light}$ \citep{joachims1999making}.
In the degenerate case with $c_j(x)\in\{0,1\}$, the kernel $\tilde{k}$ would be supported over those sample pairs belonging to the same cluster.

%Although the dual problem \eqref{soft-partial-dual} is a quadratic optimization problem,
%it does not lend itself to scalable optimization algorithms since the off-the-shelf optimization techniques for generic quadratic programs quickly
%become intractable as the number of examples increases~\citep{joachims1999making}.

%However, \eqref{soft-partial-dual} may be cast into the following standard SVM problem
%\begin{equation}
  %\sup_{\sum_{i\in\nbb_n}\alpha_i=0}\bigg\{-C\sum_{i\in\nbb_n}\ell^*(-\frac{\alpha_i}{C},y_i)-\frac{1}{2}\sum_{i,\tilde{i}\in\nbb_n}\alpha_i\alpha_{\tilde{i}}\tilde{k}(x_i,x_{\tilde{i}})\bigg\}
%\end{equation}

%, e.g., the kernel weights or weight vectors in consecutive iterations differ by a small value, or a prescribed maximum number of iterations has been reached
%The following proposition indicates that the subproblem of optimizing the objective of \eqref{soft-primal} with respect to $\{w_j^{(m)}\}_{j,m=1}^{l,M}$ for
%fixed kernel weights is related to solving a standard SVM problem.

%\begin{remark}[Solving the subproblem for fixed kernel weights]

%\hfill\qed
%\end{remark}

Next, we show that, the subproblem of optimizing the kernel weights for fixed $w_j^{(m)}$ and $b$ has a closed-form solution.
%The following proposition is a corollary of Lemma 26 in \citet{micchelli2005learning}.

\begin{proposition}[\textsc{Solution of the Subproblem w.r.t. the Kernel Weights}]\label{prop:mixture-update}
  Given fixed $w_j^{(m)}$ and $b$, the minimal $\beta_{jm}$ in optimization problem~\eqref{soft-primal} is attained for
  \begin{equation}\label{kernel-mixture-update}
    \beta_{jm}=\|w_j^{(m)}\|_2^{\frac{2}{p+1}}\Big(\sum_{k\in\nbb_M}\|w_j^{(k)}\|_2^{\frac{2p}{p+1}}\Big)^{-\frac{1}{p}}.
  \end{equation}
\end{proposition}

%\beta_{jm}=\frac{\displaystyle\|w_j^{(m)}\|_2^{\frac{2}{p+1}}}{\bigg(\sum_{k\in\nbb_M}\|w_j^{(k)}\|_2^{\frac{2p}{p+1}}\bigg)^{\frac{1}{p}}}.
We defer the detailed proof to Supplemental Material \ref{supp:proof_pontil} due to lack of space.
To apply Proposition \ref{prop:mixture-update} for updating $\beta_{jm}$, we need to compute the norm of $w_j^{(m)}$,
and this can be accomplished by the following representation of $w_j^{(m)}$ given fixed $\beta_{jm}$: (cf. Supplemental Material \ref{supp:partial_dual})
\begin{equation}\label{weight-partial-representation-text}
  w_j^{(m)}=\beta_{jm}\sum_{i\in\nbb_n}\alpha_ic_j(x_i)\phi_m(x_i).
\end{equation}
The prediction function is then derived by plugging the above representation into Eq. \eqref{decision-function}.
%Furthermore, note that the prediction function becomes
%$$
%  f(x)=\sum_{j\in\nbb_l}c_j(x)\sum_{m\in\nbb_M}\langle w_j^{(m)},\phi_m(x)\rangle+b=\sum_{j\in\nbb_l}c_j(x)\sum_{m\in\nbb_M}\beta_{jm}\sum_{i\in\nbb_n}c_j(x_i)\alpha_ik_m(x,x_i)+b.
%$$

%\begin{equation}\label{weight-l2-norm}
%  \big\Vert w_j^{(m)}\big\Vert^2=\beta_{jm}^2\sum_{i,\tilde{i}\in\nbb_n}\alpha_i\alpha_{\tilde{i}}c_j(x_i)c_j(x_{\tilde{i}})k_m(x_i,x_{\tilde{i}})
%\end{equation}
%\begin{align*}
%  f(x)&=\sum_{j\in\nbb_l}c_j(x)\sum_{m\in\nbb_M}\langle w_j^{(m)},\phi_m(x)\rangle+b\\
%  &=\sum_{j\in\nbb_l}c_j(x)\sum_{m\in\nbb_M}\Big\langle\beta_{jm}\sum_{i\in\nbb_n}\alpha_ic_j(x_i)\phi_m(x_i),\phi_m(x)\Big\rangle+b\\
%  &=\sum_{j\in\nbb_l}c_j(x)\sum_{m\in\nbb_M}\beta_{jm}\sum_{i\in\nbb_n}c_j(x_i)\alpha_ik_m(x,x_i)+b.
%\end{align*}

% convex localized MKL wrapper-based training algorithmThe analytical updates of $\beta$ and the SVM computations are optimized alternatingly.

The resulting optimization algorithm for CLMKL is shown in Algorithm~\ref{algorithm:wrapper}.
The algorithm alternates between solving an SVM subproblem for fixed kernel weights (Line 4) and
updating the kernel weights in a closed-form manner (Line 6). To improve the efficiency, we start with a crude precision and gradually improve the precision of solving the SVM subproblem.
The proposed optimization approach can potentially be extended to an interleaved algorithm where the optimization
of the MKL step is directly integrated into the SVM solver. Such a strategy can increase
the computational efficiency by up to 1-2 orders of magnitude (cf.~\citep{sonnenburg2006large} Figure 7 in \citet{kloft2011lp}).
The requirement to compute the kernel $\tilde{k}$ at each iteration can be further relaxed by updating only some randomly selected kernel elements.

%Currently, we use a similar strategy by improving the precision of solving (5) as the iteration proceeds and this works well in practice. We also foresee the strategy of randomly calculating kernel elements a promising way to further improve efficiency and will consider it in future study. We will add a corresponding discussion to the final version of the paper.

\begin{algorithm2e}[htbp]
\small
\SetKwInOut{Input}{input}
  \caption{\small Training algorithm for convex localized multiple kernel learning (CLMKL).\label{algorithm:wrapper}}
    \Input{examples $\{(x_i,y_i)_{i=1}^n\}\subset\big(\mathcal{X}\times\{-1,1\}\big)^n$ together with the likelihood functions $\{c_j(x)\}_{j=1}^l$, $M$ base kernels $k_1,\ldots,k_M$.}
    \BlankLine
    initialize $\beta_{jm}=\sqrt[p]{1/M}, w_j^{(m)}=0$ for all $j\in\nbb_l,m\in\nbb_M$\\
    \While{Optimality conditions are not satisfied}{
    calculate the kernel matrix $\tilde{k}$ by Eq. \eqref{kernel-general-loss}\\
    compute $\alpha$ by solving canonical SVM with $\tilde{k}$\\
    compute $\|w_j^{(m)}\|_2^2$ for all $j,m$ with $w_j^{(m)}$ given by Eq. \eqref{weight-partial-representation-text}\\
    update $\beta_{jm}$ for all $j,m$ according to Eq. \eqref{kernel-mixture-update}\\
    }
\end{algorithm2e}

%We remark that we also derive
%Due to the lack of the space, we defer the detailed proof to the supplementary material:

An alternative strategy would be to directly optimize
\eqref{thm:soft-complete-dual} (without the need of a two-step wrapper approach).
Such an approach has been presented in \citet{sun2010multiple} in the context of $\ell_p$-norm MKL.

%-----------------------------------------------
\vspace{-0.1cm}
\subsection{Convergence Analysis of the Algorithm}
\vspace{-0.1cm}
%-----------------------------------------------

The theorem below, which is proved in  Supplemental Material \ref{supp:lemm_bcd}, shows convergence of Algorithm \ref{algorithm:wrapper}.
The core idea is to view Algorithm \ref{algorithm:wrapper} as an example of the classical block coordinate descent (BCD) method,
convergence of which is well understood.
%For simplicity, we do not consider the bias term here.

\begin{theorem}[\sc Convergence analysis of Algorithm \ref{algorithm:wrapper}]\label{thm:clmkl_convergence}
  Assume that
	\parskip0pt
  \begin{enumerate}[({B}1)]\setlength\itemsep{0em}
	  \parskip0pt
    \item the feature map $\phi_m(x)$ is of finite dimension, i.e, $\phi_m(x)\in\mathbb{R}^{e_m},e_m<\infty,\forall m\in\nbb_M$
    \item the loss function $\ell$ is convex, continuous w.r.t. the first argument and $\ell(0,y)<\infty,\forall y\in\mathcal{Y}$
    \item any iterate $\beta_{jm}$ traversed by Algorithm \ref{algorithm:wrapper} has $\beta_{jm}>0$
    \item the SVM computation in line 4 of Algorithm \ref{algorithm:wrapper} is solved exactly in each iteration.
  \end{enumerate}
  Then, any limit point of the sequence traversed by Algorithm \ref{algorithm:wrapper} minimizes the problem \eqref{soft-primal}.
\end{theorem}

%-----------------------------------------------
\vspace{-0.1cm}
\subsection{Runtime Complexity Analysis}
\vspace{-0.1cm}
%-----------------------------------------------

At each iteration of the training stage, we need $O(n^2Ml)$ operations to calculate the kernel~\eqref{kernel-general-loss}, $O(n^2n_s)$ operations to solve a standard SVM problem, $O(Mln^2_s)$ operations to calculate the norm according to the representation \eqref{weight-partial-representation-text} and $O(Ml)$ operations to update the kernel weights. Thus, the computational cost at each iteration is $O(n^2Ml)$. The time complexity at the test stage is $O(n_tn_sMl)$. Here, $n_s$ and $n_t$ are the number of support vectors and test points, respectively.

%%%%%%%%%%%%%%%%%%%%%%%%%%%%%%%%%%%%%%%%%%%%%%%%%%%%%%%%%%%%%%%%%%
\vspace{-0.1cm}
\section{Generalization Error Bounds\label{sec:bounds}}
\vspace{-0.1cm}
%%%%%%%%%%%%%%%%%%%%%%%%%%%%%%%%%%%%%%%%%%%%%%%%%%%%%%%%%%%%%%%%%%

In this section we present generalization error bounds for our approach.
%In particular, this establishes the first generalization error bounds for a localized approach to MKL.
We give a purely data-dependent bound on the generalization error, which is obtained using Rademacher complexity theory~\citep{bartlett2002rademacher}.
To start with, our basic strategy is to plug the optimal $\beta_{jm}$ established in Eq.~\eqref{kernel-mixture-update} into \eqref{soft-primal},
so as to equivalently rewrite \eqref{soft-primal} as a block-norm regularized problem as follows:
\begin{equation}\label{convx-local-MKL-equivalent}
  \min_{w,b}\;\frac{1}{2}\sum_{j\in\nbb_l}\Big[\sum_{m\in\nbb_M}\|w_j^{(m)}\|_2^{\frac{2p}{p+1}}\Big]^{\frac{p+1}{p}}\!+\thinspace\!C\sum_{i\in\nbb_n}\ell\Big(\sum_{j\in\nbb_l}c_j(x_i)\big[\sum_{m\in\nbb_M}\langle w_j^{(m)},\!\phi_m(x_i)\rangle\big]+b,y_i\Big).
\end{equation}
Solving \eqref{convx-local-MKL-equivalent} corresponds to empirical risk minimization in the following hypothesis space:
\begin{multline*}
  H_{p,D}:=H_{p,D,M}=\bigg\{f_w:x\to\sum_{j\in\nbb_l}c_j(x_i)\big[\sum_{m\in\nbb_M}\langle w_j^{(m)},\phi_m(x_i)\rangle\big]: \sum_{j\in\nbb_l}\|w_j\|^2_{2,\frac{2p}{p+1}}\leq D\bigg\}.
\end{multline*}

The following theorem establishes the Rademacher complexity bounds for the function class $H_{p,D}$, from which we derive generalization error bounds for CLMKL in Theorem \ref{thm:genBound}. The proofs of the Theorems \ref{thm:empirical-rademacher-bound}, \ref{thm:genBound} are given in Supplemental Material \ref{supp:proof_rad}. %Denote $\bar{p}=\frac{2p}{p+1}$ for any $p\geq1$ and observe that $\bar{p}\leq2$, which implies $\bar{p}^*\geq2$.
\begin{definition}
  For a fixed sample $S=(x_1,\ldots,x_n)$, the empirical Rademacher complexity of a hypothesis space $H$ is defined as
  $$\hat{R}_n(H):=\ebb_{\bm{\sigma}}\sup_{f\in H}\frac{1}{n}\sum_{i\in\nbb_n}\sigma_if(x_i),$$
	where the expectation is taken w.r.t. $\bm{\sigma}=(\sigma_1,\ldots,\sigma_n)^\top$ with $\sigma_i,i\in\nbb_n$, being a sequence of independent uniform $\{\pm1\}$-valued random variables.
\end{definition}

\begin{theorem}[\sc CLMKL Rademacher complexity bounds\label{thm:empirical-rademacher-bound}]
  The empirical Rademacher complexity of $H_{p,D}$ can be controlled by
  \begin{equation}\label{empirical-rademacher-bound}
    \hat{R}_n(H_{p,D})\leq\frac{\sqrt{D}}{n}\inf_{2\leq t\leq\frac{2p}{p-1}}\Bigg(t\sum_{j\in\nbb_l}\Big\|\Big(\sum_{i\in \nbb_n}c_j^2(x_i)k_m(x_i,x_i)\Big)_{m=1}^M\Big\|_{\frac{t}{2}}\Bigg)^{1/2}.
  \end{equation}
  If, additionally, $k_m(x,x)\leq B$ for any $x\in\mathcal{X}$ and any $m\in\nbb_M$, then we have
  $$
  \hat{R}_{n}(H_{p,D})\leq\frac{\sqrt{DB}}{n}\inf_{2\leq t\leq\frac{2p}{p-1}}\Big(tM^{\frac{2}{t}}\sum_{j\in\nbb_l}\sum_{i\in\nbb_n}c_j^2(x_i)\Big)^{1/2}.
  $$
\end{theorem}

\begin{theorem}[\sc CLMKL Generalization Error Bounds]\label{thm:genBound}
  Assume that $k_m(x,x)\leq B,\forall m\in\nbb_M,x\in\mathcal{X}$. Suppose the loss function $\ell$ is $L$-Lipschitz and bounded by $B_{\ell}$. %satisfying $|\ell(t_1,y)-\ell(t_2,y)|\leq L|t_1-t_2|,\forall y\in\ycal$.	Denote $B_{\ell}:=\sup_{(x,y,h)}\ell(y,h(x))$.
  Then, the following inequality holds with probability larger than $1-\delta$ over samples of size $n$ for all classifiers $h\in H_{p,D}$:
	\vspace{-0.1cm}
$$\mathcal{E}_{\ell}(h)\leq\mathcal{E}_{\ell,\mathbf{z}}(h)+B_{\ell}\sqrt{\frac{\log(2/\delta)}{2n}}+2\frac{\sqrt{DB}}{n}\inf_{2\leq t\leq\frac{2p}{p-1}}\Big(tM^{\frac{2}{t}}\big[\sum_{j\in\nbb_l}\sum_{i\in\nbb_n}c^2_j(x_i)\big]\Big)^{1/2},
  \vspace{-0.1cm}
$$
  where $\mathcal{E}_{\ell}(h):=\ebb[\ell(h(x),y)]$ and $\mathcal{E}_{\ell,\mathbf{z}}(h):=\frac{1}{n}\sum_{i\in\nbb_n}\ell(h(x_i),y_i)$.% are the expected and empirical risks, respectively.
\end{theorem}

%\begin{remark}[\textsc{Tightness of Bounds}]\rm
The above bound enjoys	a mild dependence on the number of kernels.
One can show (cf. Supplemental Material \ref{supp:proof_rad}) that the dependence is $O(\log M)$ for $p\leq(\log M-1)^{-1}\log M$ and $O(M^{\frac{p-1}{2p}})$ otherwise.
In particular, the dependence is logarithmically for $p=1$ (sparsity-inducing CLMKL).
These dependencies recover the best known results for global MKL algorithms in \citet{cortes2010generalization,KloBla11,kloft2011lp}.%\hfill\qed
%\end{remark}

%\begin{remark}[\textsc{Interpretation of Bounds}]\rm
The bounds of Theorem \ref{thm:empirical-rademacher-bound} exhibit a strong dependence on the likelihood functions, which inspires us to derive a new algorithmic strategy as follows.
Consider the special case where $c_j(x)$ takes values in $\{0,1\}$ (hard cluster membership assignment),
and thus the term determining the bound has $\sum_{j\in\nbb_l}\sum_{m\in\nbb_M}c_j^2(x_i)=n$.
On the other hand, if $c_j(x)\equiv \frac{1}{l},j\in\nbb_l$ (uniform cluster membership assignment), we have the favorable term $\sum_{j\in\nbb}\sum_{i\in\nbb_n}c^2_j(x_i)=\frac{n}{l}$.
This motivates us to introduce a parameter $\tau$ controlling the complexity of the bound by considering likelihood functions of the form
\begin{equation}\label{clustering-function}
  c_j(x)\propto\exp(-\tau\text{dist}^2(x,S_j)),
\end{equation}
where $\text{dist}(x,S_j)$ is the distance between the example $x$ and the cluster $S_j$.
By letting $\tau=0$ and $\tau=\infty$, we recover uniform and hard cluster assignments, respectively.
Intermediate values of $\tau$ correspond to more balanced cluster assignments.
As illustrated by Theorem \ref{thm:empirical-rademacher-bound}, by tuning $\tau$ we optimally adjust the resulting models' complexities.%\hfill\qed
%\end{remark}

%Theorem \ref{thm:empirical-rademacher-bound} also suggests that the generalization performance of CLMKL is determined by a weighted summation of the diagonal elements in the matrix $\Sigma$,
%with weights being proportional to the trace of the gram matrix on the associated clusters.
%This fact could motivate us to optimize the correlation matrix simultaneously by introducing into the objective function \eqref{soft-primal}
%an additional regularizer $\Omega(\Sigma):=\sum_{j\in\nbb_l}\Sigma_{jj}\big(\sum_{m\in\nbb_M}\sum_{i\in S_j}k_m(x_i,x_i)\big)$.
%Notice that such formulation is no longer convex. We will consider this problem in our future research.
%\marginpar{\textcolor{red}{I think this is convex!?}}
%\end{remark}

%%%%%%%%%%%%%%%%%%%%%%%%%%%%%%%%%%%%%%%%%%%%%%%%%%%%%%%%%%%%%%%%%%
\section{Empirical Analysis and Applications\label{sec:empirical}}
%%%%%%%%%%%%%%%%%%%%%%%%%%%%%%%%%%%%%%%%%%%%%%%%%%%%%%%%%%%%%%%%%%

%-----------------------------------------------
\vspace{-0.1cm}
\subsection{Experimental Setup}
\vspace{-0.1cm}
%-----------------------------------------------
%The parameter $\tau$ appearing in \eqref{clustering-function}, which by Theorem \ref{thm:empirical-rademacher-bound} controls the complexity of the resulting models,
%is tuned by binary search as detailed in Supplemental Material \ref{supp:details_exp_setup}.
We implement the proposed convex localized MKL (CLMKL) algorithm in MATLAB and solve the involved canonical SVM problem with LIBSVM \citep{chang2011libsvm}.
The clusters $\{S_1,\ldots,S_l\}$ are computed through kernel k-means \citep[e.g.,][]{dhillon2004kernel},
but in principle other clustering methods (including convex ones such as \citet{hocking2011clusterpath}) could be used.
To further diminish k-means' potential fluctuations (which are due to random initialization of the cluster means), we repeat kernel k-means $t$ times, and choose the one with minimal clustering error (the summation of the squared distance between the examples and the associated nearest cluster) as the final partition $\{S_1,\ldots,S_l\}$.
To tune the parameter $\tau$ in \eqref{clustering-function} in a uniform manner, we introduce the notation
$$\np(\tau):=\frac{1}{nl}\sum_{i\in\nbb_n}\sum_{j\in\nbb_l}\frac{\exp(-\tau\text{dist}^2(x_i,S_j))}{\max_{\tilde{j}\in\nbb_l}\exp(-\tau\text{dist}^2(x_i,S_{\tilde{j}}))}$$
to measure the average evenness (or average excess over hard partition) of the likelihood function.
It can be checked that $\np(\tau)$ is a strictly decreasing function of $\tau$, taking value $1$ at the point $\tau=0$ and $l^{-1}$ at the point $\tau=\infty$.
Instead of tuning the parameter $\tau$ directly, we propose to tune the average excess/evenness over a subset in $[l^{-1},1]$.
The associated parameter $\tau$ are then fixed by the standard binary search algorithm.

We compare the performance attained by the proposed CLMKL to regular localized MKL (LMKL) \citep{gonen2008localized},  localized MKL based on hard clustering (HLMKL) \citep{lei2015theory},
the SVM using a uniform kernel combination (UNIF) \citep{Cor09}, and $\ell_p$-norm MKL \citep{kloft2011lp},
which includes classical MKL \citep{lanckriet2004learning} as a special case. We optimize $\ell_p$-norm MKL and CLMKL until the relative duality gap drops below $0.001$.
The calculation of the gradients in LMKL \citep{gonen2008localized} requires $O(n^2M^2d)$ operations,
which scales poorly, and the definition of the gating model requires the information of primitive features,
which is not available for the biological applications studied below, all of which involve string kernels.
%We therefore employ our own implementation of LMKL.controls the complexity of the resulting models
In Supplemental Material \ref{supp:fastGoenen}, we therefore give a fast and general formulation of LMKL,
which requires only $O(n^2M)$ operations per iteration.
Our implementation of which is available from the following webpage,
together with our CLMKL implementation and scripts to reproduce
the experiments:

{\centering\small\url{https://www.dropbox.com/sh/hkkfa0ghxzuig03/AADRdtSSdUSm8hfVbsdjcRqva?dl=0}}.

In the following we report detailed results for various real-world experiments. Further details are shown
in Supplemental Material \ref{supp:exp_details}.

\vspace{-0.1cm}
\subsection{Splice Site Recognition}
\vspace{-0.1cm}
%-----------------------------------------------

Our first experiment aims at detecting splice sites in the organism \emph{Caenorhabditis elegans}, which is an important task in computational gene finding
as splice sites are located on the DNA strang right at the boundary of exons (which code for proteins) and introns (which do not).
%The more accurately a splice site can be located, the easier and more reliable it becomes to locate the genes on the DNA.
We experiment on the \texttt{mkl-splice} data set, which we download from {\small\url{http://mldata.org/repository/data/viewslug/mkl-splice/}}.
It includes 1000 splice site instances and 20 weighted-degree kernels with degrees ranging from 1 to 20 \citep{BenOngSonSchRae08}.
The experimental setup for this experiment is as follows.
We create random splits of this dataset into training set, validation set and test set, with size of training set traversing over the set $\{50,100,200,300,\ldots,800\}$.
We apply kernel-kmeans with uniform kernel to generate a partition with $l=3$ clusters for both CLMKL and HLMKL, and use this kernel to define the gating model in LMKL.
To be consistent with previous studies, we use the area under the ROC curve (AUC) as an evaluation criterion.
We tune the SVM regularization parameter from $10^{\{-1,-0.5,\ldots,2\}}$, and the average evenness over the interval $[0.4,0.8]$ with eight linearly equally spaced points, based on the AUCs on the validation set.
All the base kernel matrices are multiplicatively normalized before training.
We repeat the experiment $50$ times, and report mean AUCs on the test set as well as standard deviation.
%The experimental setup of this experiment is described in Supplemental Material ~\ref{supp:splice}.
Figure \ref{fig:splice_TSS} (a) shows the results as a function of the training set size $n$.
\begin{figure}[!t]
\centering
  \subfigure[\texttt{Splice}]{\includegraphics[width=0.4\textwidth]{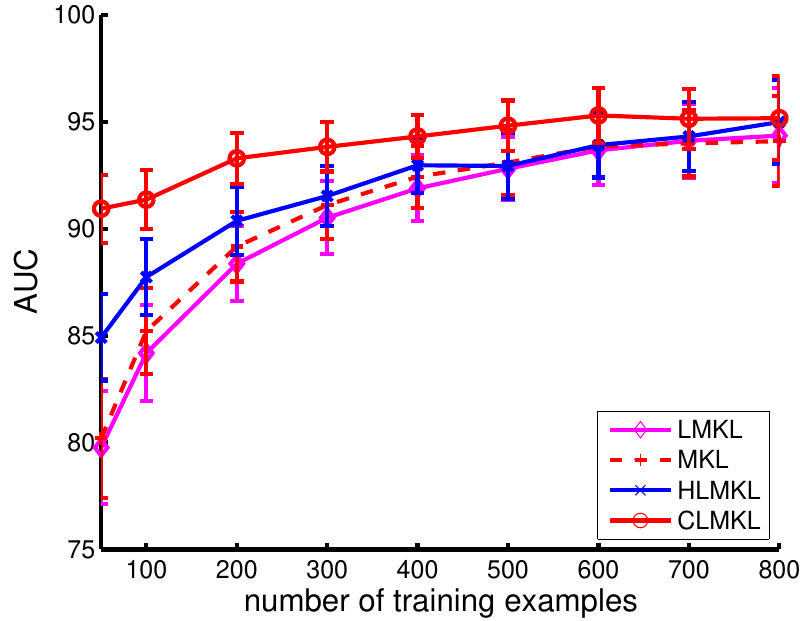}}
  \subfigure[\texttt{TSS}]{\includegraphics[width=0.4\textwidth]{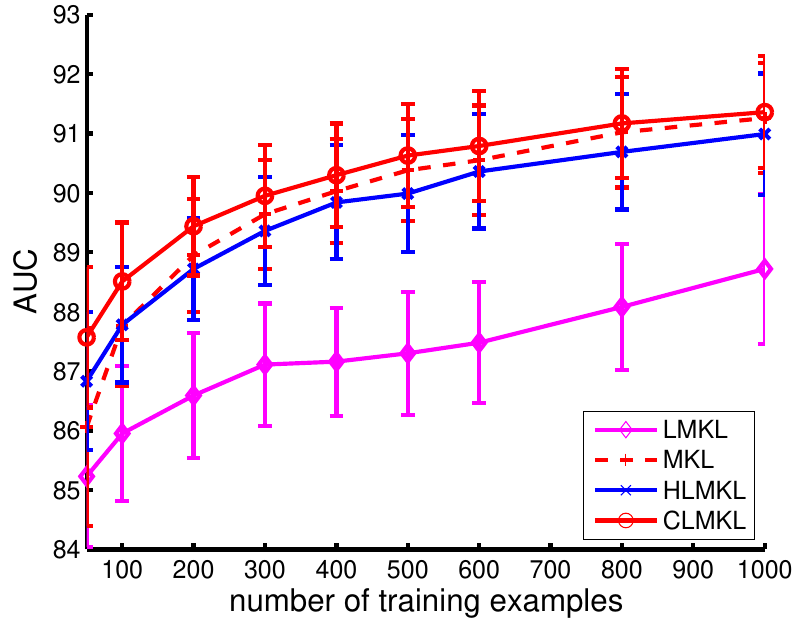}}
	\vspace{-0.3cm}
\caption{
Results of the gene finding experiments: splice site detection (left) and transcription start site detection (right). To clean the presentation, results for UNIF are not given here. The parameter $p$ for CLMKL, HLMKL and MKL is set as $1$ here.}%Errorbars are given in Supplemental Material \ref{supp:tss}
%AUC are shown as a function of the training set size for LMKL, UNIF, $\ell_p$ MKL and CLMKL.
\label{fig:splice_TSS}
\vspace{-0.3cm}
\end{figure}

We observe that CLMKL achieves, for all $n$, a significant gain over all baselines.
This improvement is especially strong for small $n$. For $n=50$, CLMKL attains $90.9\%$ accuracy,
while the best baseline only achieves $85.4\%$, improving by $5.5\%$.  Detailed results with standard deviation are reported in Table \ref{tab:result-splice}. A hypothetical explanation of the improvement from CLMKL is that splice sites are characterized by nucleotide sequences---so-called \emph{motifs}---the length of which may differ from site to site \citep{sonnenburg2008poims}.
The 20 employed kernels count matching subsequences of length 1 to 20, respectively.
For sites characterized by smaller motifs, low-degree WD-kernels are thus more effective than high-degree ones, and vice versa for sites containing longer motifs.
%in Supplemental Material \ref{supp:tss}

\begin{table*}[!h]
\scriptsize
\setlength{\tabcolsep}{1pt}
\centering
  \begin{tabular}{*{10}{|c}|}\hline
  &$50$&$100$&$200$&$300$&$400$&$500$&$600$&$700$&$800$\\ \hline
  UNIF&$79.5\!\pm\!2.8\bullet$&$ 84.2\!\pm\!2.2\bullet$&$ 88.0\!\pm\!1.7\bullet$&$ 90.0\!\pm\!1.7\bullet$&$ 91.6\!\pm\!1.5\bullet$&$ 92.4\!\pm\!1.5\bullet$&$ 93.3\!\pm\!1.7\bullet$&$ 93.6\!\pm\!1.7\bullet$&$ 93.8\!\pm\!2.3\bullet$\\ \hline
  LMKL&$79.8\!\pm\!2.7\bullet$& $84.2\!\pm\!2.3\bullet$& $88.4\!\pm\!1.7\bullet$& $90.5\!\pm\!1.7\bullet$& $91.9\!\pm\!1.5\bullet$& $92.8\!\pm\!1.5\bullet$& $93.7\!\pm\!1.6\bullet$& $94.1\!\pm\!1.7\bullet$& $94.3\!\pm\!2.2\bullet$\\ \hline
  MKL,p=1&$80.2\!\pm\!2.8\bullet$&$ 85.2\!\pm\!2.0\bullet$&$ 89.2\!\pm\!1.6\bullet$&$ 91.1\!\pm\!1.6\bullet$&$ 92.5\!\pm\!1.5\bullet$&$ 93.1\!\pm\!1.5\bullet$&$ 93.9\!\pm\!1.4\bullet$&$ 94.0\!\pm\!1.6\bullet$&$ 94.2\!\pm\!2.1\bullet$\\\hline
  MKL,p=2&$79.6\!\pm\!2.8\bullet$&$ 84.3\!\pm\!2.2\bullet$&$ 88.3\!\pm\!1.7\bullet$&$ 90.4\!\pm\!1.6\bullet$&$ 91.8\!\pm\!1.5\bullet$&$ 92.5\!\pm\!1.5\bullet$&$ 93.4\!\pm\!1.6\bullet$&$ 93.6\!\pm\!1.6\bullet$&$ 93.8\!\pm\!2.2\bullet$\\\hline
  MKL,p=1.33&$79.7\!\pm\!2.9\bullet$&$ 84.6\!\pm\!2.1\bullet$&$ 88.6\!\pm\!1.7\bullet$&$ 90.6\!\pm\!1.6\bullet$&$ 92.0\!\pm\!1.5\bullet$&$ 92.7\!\pm\!1.5\bullet$&$ 93.5\!\pm\!1.5\bullet$&$ 93.7\!\pm\!1.6\bullet$&$ 93.8\!\pm\!2.1\bullet$\\\hline
  HLMKL,p=1&$84.9\!\pm\!2.0\bullet$& $87.7\!\pm\!1.8\bullet$& $90.4\!\pm\!1.6\bullet$& $91.5\!\pm\!1.4\bullet$& $93.0\!\pm\!1.3\bullet$& $92.9\!\pm\!1.6\bullet$& $93.9\!\pm\!1.5\bullet$& $94.3\!\pm\!1.6\bullet$& $95.0\!\pm\!2.0\bullet$\\ \hline
  HLMKL,p=2&$84.9\!\pm\!2.0\bullet$& $87.0\!\pm\!1.7\bullet$& $90.4\!\pm\!1.4\bullet$& $91.1\!\pm\!1.6\bullet$& $92.6\!\pm\!1.4\bullet$& $93.5\!\pm\!1.6\bullet$& $94.7\!\pm\!1.4\bullet$& $94.6\!\pm\!1.4\bullet$& $94.4\!\pm\!2.2\bullet$\\ \hline
  HLMKL,p=1.33&$85.4\!\pm\!1.9\bullet$& $88.5\!\pm\!1.7\bullet$& $90.1\!\pm\!1.6\bullet$& $91.7\!\pm\!1.4\bullet$& $92.7\!\pm\!1.2\bullet$& $93.4\!\pm\!1.6\bullet$& $94.6\!\pm\!1.5\bullet$& $94.4\!\pm\!1.7\bullet$& $94.4\!\pm\!2.1\bullet$\\ \hline
  CLMKL,p=1&$90.9\!\pm\!1.6$& $91.3\!\pm\!1.4\bullet$& $93.3\!\pm\!1.2\bullet$& $93.8\!\pm\!1.2\bullet$& $94.3\!\pm\!1.0\bullet$& $94.8\!\pm\!1.2$& $95.3\!\pm\!1.3\bullet$& $95.1\!\pm\!1.4\bullet$& $95.2\!\pm\!2.0\bullet$\\\hline
  CLMKL,p=2&$90.5\!\pm\!1.6\bullet$& $92.3\!\pm\!1.2\bullet$& $93.0\!\pm\!1.2\bullet$& $94.0\!\pm\!1.2$& $94.4\!\pm\!1.1\bullet$& $94.7\!\pm\!1.2\bullet$& $95.4\!\pm\!1.4\bullet$& $95.3\!\pm\!1.5\bullet$& $95.6\!\pm\!1.9\bullet$\\\hline
  CLMKL,p=1.33&$90.9\!\pm\!1.5$& $90.1\!\pm\!1.3$& $92.7\!\pm\!1.2$& $94.1\!\pm\!1.2$& $94.8\!\pm\!1.1$& $94.9\!\pm\!1.1$& $95.6\!\pm\!1.2$& $95.4\!\pm\!1.5$& $95.4\!\pm\!1.9$\\\hline
  \end{tabular}
    \caption{Performances achieved by LMKL, UNIF, regular $\ell_p$ MKL, HLMKL, and CLMKL on Splice Dataset. $\bullet$ indicates that CLMKL with $p=1.33$ is significantly better than the compared method (paired t-tests at $95\%$ significance level).\label{tab:result-splice}}
\end{table*}

%-----------------------------------------------
\vspace{-0.1cm}
\subsection{Transcription Start Site Detection}
\vspace{-0.1cm}
%-----------------------------------------------

Our next experiment aims at detecting transcription start sites (TSS) of RNA Polymerase II binding genes in genomic DNA sequences. We experiment on the \texttt{TSS} data
set, which we downloaded from {\small\url{http://mldata.org/repository/data/viewslug/tss/}}.  This data set, which is included in the larger study of
\cite{SonZieRae06}, comes with $5$ kernels. %, capturing various complementary aspects: a weighted-degree kernel
%\cite{wdkernel}
%representing the TSS signal (\texttt{TSS}), two spectrum kernels %\cite{LesEskNob02}
%around the promoter region (\texttt{Promo}) and the 1st exon (\texttt{1st Ex}), respectively, and two linear kernels based on twisting angles (\texttt{Angle}) and stacking energies (\texttt{Energ}), respectively.
The SVM based on the uniform combination of these $5$ kernels was found to have the highest overall performance among $19$ promoter prediction programs \citep{AbeelPS09}.
It therefore constitutes a strong baseline.
To be consistent with previous studies \citep{AbeelPS09,kloft2011lpb,SonZieRae06}, we use the area under the ROC curve (AUC) as an evaluation criterion.
We consider the same experimental setup as in the splice detection experiment.
The gating function and the partition are computed with the TSS kernel, which carries most of the discriminative information \citep{SonZieRae06}.
All kernel matrices were normalized with respect to their trace, prior to the experiment.

\begin{table*}[!h]
\scriptsize
\setlength{\tabcolsep}{1pt}
\centering
  \begin{tabular}{*{10}{|c}|}\hline
  &$50$&$100$&$200$&$300$&$400$&$500$&$600$&$800$&$1000$\\ \hline
  UNIF&$83.9\!\pm\!2.4\bullet$& $86.2\!\pm\!1.3\bullet$& $87.6\!\pm\!1.0\bullet$& $88.4\!\pm\!0.9\bullet$& $88.7\!\pm\!0.9\bullet$& $89.1\!\pm\!0.9\bullet$& $89.2\!\pm\!1.0\bullet$& $89.6\!\pm\!1.1\bullet$& $89.8\!\pm\!1.1\bullet$\\\hline
  LMKL&$85.2\!\pm\!1.2\bullet$& $85.9\!\pm\!1.1\bullet$& $86.6\!\pm\!1.1\bullet$& $87.1\!\pm\!1.0\bullet$& $87.2\!\pm\!0.9\bullet$& $87.3\!\pm\!1.0\bullet$& $87.5\!\pm\!1.0\bullet$& $88.1\!\pm\!1.1\bullet$& $88.7\!\pm\!1.3\bullet$\\\hline
  MKL,p=1&$86.0\!\pm\!1.7\bullet$& $87.7\!\pm\!1.0\bullet$& $88.9\!\pm\!0.9\bullet$& $89.6\!\pm\!0.9\bullet$& $90.0\!\pm\!0.9\bullet$& $90.3\!\pm\!0.9\bullet$& $90.5\!\pm\!0.9$& $91.0\!\pm\!0.9$& $91.2\!\pm\!0.9$\\\hline
  MKL,p=2&$85.1\!\pm\!2.0\bullet$& $86.9\!\pm\!1.1\bullet$& $88.1\!\pm\!0.9\bullet$& $88.8\!\pm\!0.9\bullet$& $89.2\!\pm\!0.9\bullet$& $89.6\!\pm\!0.9\bullet$& $89.8\!\pm\!1.0\bullet$& $90.3\!\pm\!1.0\bullet$& $90.7\!\pm\!0.9\bullet$\\\hline
  MKL,p=1.33&$85.7\!\pm\!1.8\bullet$& $87.5\!\pm\!1.0\bullet$& $88.7\!\pm\!0.9\bullet$& $89.4\!\pm\!0.9\bullet$& $89.8\!\pm\!0.9\bullet$& $90.2\!\pm\!0.9\bullet$& $90.4\!\pm\!0.9\bullet$& $90.9\!\pm\!0.9\bullet$& $91.2\!\pm\!0.9\bullet$\\\hline
  HLMKL,p=1&$86.8\!\pm\!1.2\bullet$& $87.8\!\pm\!1.0\bullet$& $88.7\!\pm\!0.9\bullet$& $89.4\!\pm\!0.9\bullet$& $89.8\!\pm\!1.0\bullet$& $90.0\!\pm\!1.0\bullet$& $90.4\!\pm\!1.0\bullet$& $90.7\!\pm\!1.0\bullet$& $91.0\!\pm\!1.0\bullet$\\ \hline
  HLMKL,p=2&$86.3\!\pm\!1.4\bullet$& $87.5\!\pm\!1.0\bullet$& $88.5\!\pm\!0.9\bullet$& $89.3\!\pm\!0.9\bullet$& $89.4\!\pm\!0.9\bullet$& $89.7\!\pm\!0.9\bullet$& $89.8\!\pm\!1.0\bullet$& $90.3\!\pm\!1.1\bullet$& $90.5\!\pm\!1.0\bullet$\\ \hline
  HLMKL,p=1.33&$86.5\!\pm\!1.4\bullet$& $87.7\!\pm\!1.1\bullet$& $88.7\!\pm\!0.9\bullet$& $89.3\!\pm\!0.9\bullet$& $89.8\!\pm\!1.0\bullet$& $90.1\!\pm\!0.9\bullet$& $90.2\!\pm\!1.0\bullet$& $90.7\!\pm\!1.0\bullet$& $91.0\!\pm\!0.9\bullet$\\ \hline
  CLMKL,p=1&$87.6\!\pm\!1.2$& $88.5\!\pm\!1.0$& $89.4\!\pm\!0.8$& $90.0\!\pm\!0.9$& $90.3\!\pm\!0.9\bullet$& $90.6\!\pm\!0.9$& $90.8\!\pm\!0.9\bullet$& $91.2\!\pm\!0.9\bullet$& $91.4\!\pm\!0.9$\\\hline
  CLMKL,p=2&$87.3\!\pm\!1.3\bullet$& $88.3\!\pm\!1.0\bullet$& $89.1\!\pm\!0.8\bullet$& $89.6\!\pm\!0.8\bullet$& $89.9\!\pm\!0.9\bullet$& $90.2\!\pm\!0.9\bullet$& $90.3\!\pm\!0.9\bullet$& $90.7\!\pm\!1.0\bullet$& $90.9\!\pm\!0.9\bullet$\\\hline
  CLMKL,p=1.33&$87.6\!\pm\!1.2$& $88.6\!\pm\!0.9$& $89.4\!\pm\!0.8$& $89.9\!\pm\!0.9$& $90.2\!\pm\!0.9$& $90.5\!\pm\!0.9$& $90.6\!\pm\!1.0$& $91.1\!\pm\!1.0$& $91.3\!\pm\!0.9$\\\hline
  \end{tabular}
    \caption{Performances achieved by LMKL, UNIF, regular $\ell_p$ MKL, HLMKL and CLMKL on TSS Dataset. $\bullet$ indicates that CLMKL with $p=1.33$ is significantly better than the compared method (paired t-tests at $95\%$ significance level).
\label{tab:result-tss}}
\end{table*}

Figure \ref{fig:splice_TSS} (b) shows the AUCs on the test data sets as a function of the number of training examples.%(for detailed numerical values cf. Table \ref{tab:result-tss} in Supplemental Material \ref{supp:splice})
We observe that CLMKL attains a consistent improvement over other competing methods. Again, this improvement is most significant when $n$ is small. Detailed results with standard deviation are reported in Table \ref{tab:result-tss}.%the number of training examples

%-----------------------------------------------
\vspace{-0.1cm}
\subsection{Protein Fold Prediction}\label{sec:bio}%---An Application from the Domain of Computational Biology
\vspace{-0.1cm}
%-----------------------------------------------

\begin{table*}[!b]
\scriptsize
\setlength{\tabcolsep}{1pt}
\centering
  \begin{tabular}{*{12}{|c}|}\hline
    &\multirow{2}{*}{UNIF}&\multirow{2}{*}{LMKL}&\multicolumn{3}{c|}{MKL}&\multicolumn{3}{c|}{HLMKL}&\multicolumn{3}{c|}{CLMKL}\\ \cline{4-12}
        &&&$p=1$&$p=1.2$&$p=2$&$p=1$&$p=1.2$&$p=2$&$p=1$&$p=1.2$&$p=2$\\\hline
    ACC&$68.4\bullet$&$64.3\bullet$&$68.7\bullet$&$74.2\bullet$&$70.8\bullet$&$72.7\pm1.3\bullet$&$74.6\pm0.6$&$72.4\pm0.8\bullet$&   $71.3\pm0.5\bullet$&$\bm{75.0\pm0.7}$&$71.7\pm0.5\bullet$\\\hline
  \end{tabular}
    \caption{Results of the protein fold prediction experiment. $\bullet$ indicates that CLMKL with $p=1.2$ is significantly better than the compared method (paired t-tests at $95\%$ significance level).
		%accuracies achieved by LMKL, UNIF, regular $\ell_p$ MKL and the proposed convex localized MKL (CLMKL) on the protein folding class prediction task.
		\label{tab:result-dingshen}}
		\vspace{-0.3cm}
\end{table*}

Protein fold prediction is a key step towards understanding the function of proteins, as the folding class of a protein is closely linked with its function; thus it is crucial for drug design.
We experiment on the protein folding class prediction dataset by \citet{ding2001multi}, which was also used in \citet{campbell2011learning,kloft2011lpb,KloBla11}.
This dataset consists of $27$ fold classes with $311$ proteins used for training and $383$ proteins for testing.
We use exactly the same $12$ kernels as in \citet{campbell2011learning,kloft2011lpb,KloBla11} reflecting different features, such as
van der Waals volume, polarity and hydrophobicity.
We precisely replicate the experimental setup of previous experiments by \citet{campbell2011learning,kloft2011lpb,KloBla11}, which is detailed in Supplementary Material \ref{supp_protein}.
We report the mean prediction accuracies, as well as standard deviations in Table~\ref{tab:result-dingshen}.

The results show that CLMKL surpasses regular $\ell_p$-norm MKL for all values of $p$,
%This is a non-sparse scenario for which \cite{kloft2011lpb} achieved the
%Note that we do not achieve the state-of-the art accuracy of $74.4\%$ reported in \cite{kloft2011lpb}, possibly due to different implementations of the $\ell_p$-norm MKL.
and achieves accuracies up to $0.6\%$ higher than the one reported in \citet{kloft2011lpb}, which is higher than the initially reported accuracies in \citet{campbell2011learning}.
%For example, CLMKL with $p=1.2$ achieves an impressive accuracy of $75.1\%$.
LMKL works poorly in this dataset, possibly because LMKL based on precomputed custom kernels requires to optimize $nM$ additional variables, which may overfit.
%It can be also seen that both majority voting and averaging work well in stabilizing the results: the standard deviations are always less than $0.7\%$ in our experiments.

%Note that localized MKL as implemented by \citet{gonen2008localized} is not applicable here as it requires the primitive features of the raw input instances,
%which are neither provided by the collector nor can they be computed: the set of kernels includes two string kernels, for which
%no input features can be given.
%
%The columns ``Voting'' and ``Aver'' show the prediction accuracy achieved by majority voting and the averaging strategy, respectively.
%The column ``Oracle'' shows the prediction accuracies of the optimal candidate model on the test data, which indicates the prediction power of a single candidate model on the test set.

%-----------------------------------------------
\vspace{-0.1cm}
\subsection{Visual Image Categorization---UIUC Sports}
\vspace{-0.1cm}
%-----------------------------------------------

We experiment on the UIUC Sports event dataset \citep{li2007andli} consisting of 1574 images, belonging to 8 image classes of sports activities.
We compute 9 $\chi^2$-kernels based on SIFT features and global color histograms,
which is described in detail in Supplemental Material \ref{supp_uiuc}, where we also give background on the experimental setup.
%Table \ref{tab:result-uiuc} shows the results.

From the results shown in Table \ref{tab:result-uiuc}, we observe that CLMKL achieves a performance improvement by $0.26\%$ over the $\ell_p$-norm MKL baseline while localized MKL as in \citet{gonen2008localized} underperforms the MKL baseline.
%Although the improvement is decent, a comparison to the best known results from the literature \citep{liu2014self} shows
%that CLMKL attains, to our best knowledge, the highest result ever achieved on the UIUC dataset.%, thus defining a new state of the art for this dataset.

\begin{table*}[!h]
\scriptsize
\setlength{\tabcolsep}{1pt}
\centering
  \begin{tabular}{*{8}{|c}|}\hline
 &    MKL & LMKL & CLMKL &   & MKL & LMKL & CLMKL\\ \hline
ACC & 90.00 & 87.29 & {90.26} & $\Delta $ & 0+\,11=\,0- & 0+\,1=\,10- & 4+\,6=\,1-\\\hline
  \end{tabular}
    \caption{Results of the visual image recognition experiment on the UIUC sports dataset.
		%Prediction accuracies achieved by LMKL, regular $\ell_p$-MKL and the proposed convex localized MKL (CLMKL) on the UIUC sports classification task.
		$\Delta$ indicates on how many outer cross validation test splits a method is worse ($n-$), equal ($n=$) or better ($n+$) than MKL.\label{tab:result-uiuc}}
\end{table*}

%\vspace{-0.1cm}
%-----------------------------------------------
\vspace{-0.1cm}
\subsection{Execution Time Experiments}\label{supp:execution}
\vspace{-0.1cm}
%-----------------------------------------------
%\begin{figure}[!h]
%\centering
%\includegraphics[width=0.46\textwidth]{execution_tss}
%\caption{Training time versus the training set size for LMKL, UNIF, $\ell_p$ MKL, HLMKL and CLMKL.
%\label{fig:execution}}
%\end{figure}

To demonstrate the efficiency of the proposed implementation, we compare the training time
for UNIF, LMKL, $\ell_p$-norm MKL, HLMKL and CLMKL on the TSS dataset.
We fix the regularization parameter
\begin{wrapfigure}{r}{0.5\textwidth}
\vspace{-0.25cm}
\includegraphics[width=0.46\textwidth]{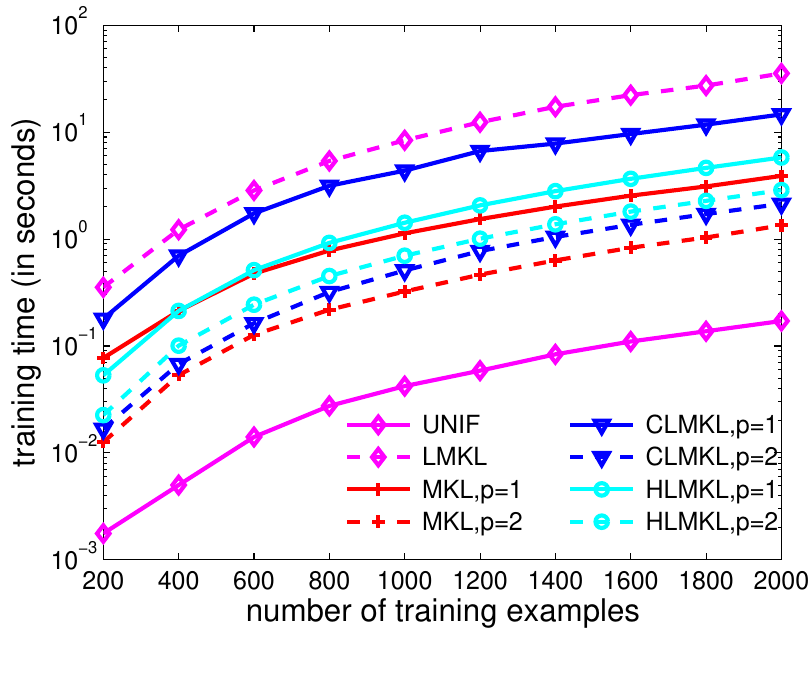}%\caption{Training time versus the training set size for LMKL, UNIF, $\ell_p$ MKL, HLMKL and CLMKL.
\vspace{-.6cm}
\end{wrapfigure}
$C=1$. We fix $l=3$ and $\np=0.5$ for CLMKL, and fix $l=3$ for HLMKL. On the image to the right, we plot the training time versus the training set size.
We repeat the experiment $20$ times and report the average training time here.
We optimize CLMKL, HLMKL and MKL until the relative gap is under $10^{-3}$.
The figure implies that CLMKL converges faster than LMKL.
Furthermore, training an $\ell_2$-norm MKL requires significantly less time than training an $\ell_1$-norm MKL,
which is consistent with the fact that the dual problem of $\ell_2$-norm MKL is much smoother than the $\ell_1$-norm counterpart.

\vspace{-0.1cm}
\section{Conclusions\label{sec:conclusions}}
\vspace{-0.1cm}
%%%%%%%%%%%%%%%%%%%%%%%%%%%%%%%%%%%%%%%%%%%%%%%%%%%%%%%%%%%%%%%%%%%

Localized approaches to multiple kernel learning allow for flexible distribution of kernel weights over the input space,
which can be a great advantage when samples require varying kernel importance.
As we show in this paper, this can be the case in image recognition and several computational biology applications.
However, almost prevalent approaches to localized MKL require solving difficult non-convex optimization problems,
which makes them potentially prone to overfitting as theoretical guarantees such as generalization error bounds are yet unknown.

In this paper, we propose a theoretically grounded approach to localized MKL,
consisting of two subsequent steps: 1. clustering the training instances and 2.
computation of the kernel weights for each cluster through a single convex optimization problem.
For which we derive an efficient optimization algorithm based on Fenchel duality.
Using Rademacher complexity theory, we establish large-deviation inequalities for localized MKL,
showing that the smoothness in the cluster membership assignments crucially controls the generalization error.
%We derived partial and complete dual representations using Fenchel conjugation theory, for which we developed an optimization algorithm, which is based on alternating a dedicated SVM solver and closed-form updates of the kernel weights.
%The theoretical analysis based on Rademacher complexity theory resulted, for the first time, in large deviation inequalities that connect the cluster membership assignment with the generalization capability of the learning algorithm.
%Computational experiments on controlled synthetic data showed that the proposed approach can achieve prediction accuracies up to $30\%$ higher than its global and non-convex local counterparts.
The proposed method is well suited for deployment in the domains of computer vision and computational biology. %: for the tasks of visual image recognition and protein fold prediction, we achieve empirical results that both surpass the best known results from the literature on the respective data sets
For splice site detection, CLMKL achieves up to $5\%$ higher accuracy than its global and non-convex localized counterparts.
%We also present a fast implementation of LMKL algorithm in \cite{gonen2008localized} based on the kernel trick.

Future work could analyze extension of the methodology to semi-supervised learning \citep{gornitz2009active,gornitz2013toward} or using different clustering objectives \citep{vogt2015probabilistic,hocking2011clusterpath} and how to principally include the construction of the data partition into our framework
by constructing partitions that can capture the local variation of prediction importance of different features.%,thus thriving to solve the clustering and MKL steps in a joint convex problem.

%Future work will aim at investigating convex clustering strategies such as \citet{hocking2011clusterpath} to overcome the k-means preprocessing step,
%and how to principally include the construction of the data partition into our framework
%by constructing partitions that can capture the local variation of prediction importance of different features,
%thus thriving to solve the clustering and MKL steps in a joint convex problem.
%Another research direction is to directly integrate the MKL step into the SVM solver, as pioneered by \cite{sonnenburg2006large}. We expect that such an implementation would lead to a speed-up in computational efficiency by up to 1-2 orders of magnitude. % or by automatically learning the graph Laplacian using appropriate matrix regularization.

%%%%%%%%%%%%%%%%%%%%%%%%%%%%%%%%%%%%%%%%%%%%%%%%%%%%%%%%%%%%%%%%%%%%
\vspace{-0.1cm}
\section*{Acknowledgments}\label{sec:ack}
\vspace{-0.1cm}
%%%%%%%%%%%%%%%%%%%%%%%%%%%%%%%%%%%%%%%%%%%%%%%%%%%%%%%%%%%%%%%%%%%%
%
YL acknowledges support from the NSFC/RGC Joint Research Scheme [RGC Project No. N\_CityU120/14 and NSFC Project No. 11461161006].
AB acknowledges support from Singapore University of Technology and Design Startup Grant SRIS15105.
MK acknowledges support from the German Research Foundation (DFG) award KL 2698/2-1 and from the Federal Ministry of Science and Education (BMBF) awards 031L0023A and 031B0187B.

%\acks{MK acknowledges support from the German Research Foundation (DFG) award KL 2698/2-1 and from the Federal Ministry of Science and Education (BMBF) awards 031L0023A and 031B0187B.}

\section*{Supplemental Material}
\appendix
\numberwithin{equation}{section}
\numberwithin{theorem}{section}
\numberwithin{figure}{section}
\numberwithin{table}{section}
\renewcommand{\thesection}{{\Alph{section}}}
\renewcommand{\thesubsection}{\Alph{section}.\arabic{subsection}}
\renewcommand{\thesubsubsection}{\Roman{section}.\arabic{subsection}.\arabic{subsubsection}}
\setcounter{secnumdepth}{-1}
\setcounter{secnumdepth}{3}
%%%%%%%%%%%%%%%%%%%%%%%%%%%%%%%%%%%%%%%%%%%%%%%%%%%%%%%%%%%%%%%%%%%

%%%%%%%%%%%%%%%%%%%%%%%%%%%%%%%%%%%%%%%%%%%%%%%%%%%%%%%%%%%%%%%%%%%
\section{Lemmata and Proofs}\label{app:proofs}
%%%%%%%%%%%%%%%%%%%%%%%%%%%%%%%%%%%%%%%%%%%%%%%%%%%%%%%%%%%%%%%%%%%

%-----------------------------------------------
\subsection{Lemmata Used for Dualization in Section \ref{sec:dual}}\label{app:proofs_dual}
%-----------------------------------------------
Let $H_1,\ldots,H_M$ be $M$ Hilbert spaces and $p\geq1$. Define the function $g_p(v_1,\ldots,v_M):H_1\times\cdots\times H_M\to\rbb$ by $$g_p(v_1,\ldots,v_M)=\frac{1}{2}\|(v_1,\ldots,v_M)\|^2_{2,p},\quad p\geq1.$$ For any $p\geq1$, denote by $p^*$ the conjugated exponent satisfying $\frac{1}{p}+\frac{1}{p^*}=1$.
\begin{lemma}\label{lem:gradient-pnorm}
  The gradient of $g_p$ is $$\frac{\partial g_p(v_1,\ldots,v_M)}{\partial v_m}=\big[\sum_{\tilde{m}\in\nbb_M}\|v_{\tilde{m}}\|_2^p\big]^{\frac{2}{p}-1}\|v_m\|_2^{p-2}v_m.$$
\end{lemma}

\begin{proof}
By the chain rule, we have
\begin{align*}
  \frac{\partial g_p(v_1,\ldots,v_M)}{\partial v_m}&=\frac{1}{p}\big[\sum_{\tilde{m}\in\nbb_M}\|v_{\tilde{m}}\|_2^p\big]^{\frac{2}{p}-1}\frac{\partial\langle v_m,v_m\rangle^{\frac{p}{2}}}{\partial v_m}\\
  &=\frac{1}{2}\big[\sum_{\tilde{m}\in\nbb_M}\|v_{\tilde{m}}\|_2^p\big]^{\frac{2}{p}-1}\frac{\partial\langle v_m,v_m\rangle}{\partial v_m}\langle v_m,v_m\rangle^{\frac{p}{2}-1}\\
  &=\big[\sum_{\tilde{m}\in\nbb_M}\|v_{\tilde{m}}\|_2^p\big]^{\frac{2}{p}-1}\|v_m\|_2^{p-2}v_m.
\end{align*}
\end{proof}

\begin{lemma}[\citealt{micchelli2005learning}]\label{lem:micchelli}
  Let $a_i\geq0,i\in\nbb_d$ and $1\leq r<\infty$. Then
  $$\min_{\eta:\eta_i\geq0,\sum_{i\in\nbb_d}\eta_i^r\leq1}\sum_{i\in\nbb_d}\frac{a_i}{\eta_i}=\Big(\sum_{i\in\nbb_d}a_i^{\frac{r}{r+1}}\Big)^{1+\frac{1}{r}}$$and the minimum is attained at $\eta_i=a_i^{\frac{1}{r+1}}\Big(\sum_{k\in\nbb_d}a_k^{\frac{r}{r+1}}\Big)^{-\frac{1}{r}}.$
\end{lemma}

%-------------------------------------
\subsection{Proof of Representer Theorem (Theorem~\ref{thm:repre})}\label{supp_proof_repr}
%-------------------------------------

\begin{proof}[Proof of Theorem~\ref{thm:repre}]
In our derivation \eqref{soft-complete-dual-derivation} of the dual problem, the variable $w_j(\alpha):=\big(w_j^{(1)}(\alpha),\ldots,w_j^{(M)}(\alpha)\big)$ should meet the optimality in the sense
$$w_j(\alpha)=\arg\max_{v_1\in H_1,\ldots,v_M\in H_M}\bigg[\Big\langle(v_m)_{m=1}^M,\Big(\sum_{i=1}^n\alpha_ic_j(x_i)\phi_m(x_i)\Big)_{m=1}^M\Big\rangle-\frac{1}{2}\|(v_m)_{m=1}^M\|^2_{2,\frac{2p}{p+1}}\bigg].$$
Since $(\bigtriangledown f)^{-1}=\bigtriangledown f^*$ for any convex function $f$ and the Fenchel-conjugate of $g_p$ is $g_{p^*}$, we obtain
\begin{align*}
  w_j(\alpha)&=\bigtriangledown g^{-1}_{\frac{2p}{p+1}}\Big(\sum_{i\in\nbb_n}\alpha_ic_j(x_i)\phi_1(x_i),\ldots,\sum_{i\in\nbb_n}\alpha_ic_j(x_i)\phi_M(x_i)\Big)\\
  &=\bigtriangledown g_{\frac{2p}{p-1}}\Big(\sum_{i\in\nbb_n}\alpha_ic_j(x_i)\phi_1(x_i),\ldots,\sum_{i\in\nbb_n}\alpha_ic_j(x_i)\phi_M(x_i)\Big)\\
  &\stackrel{\text{Lem.~\ref{lem:gradient-pnorm}}}{=}\big[\sum_{\tilde{m}\in\nbb_M}\|\sum_{i\in\nbb_n}\alpha_ic_j(x_i)\phi_{\tilde{m}}(x_i)\|_2^{\frac{2p}{p-1}}\big]^{-\frac{1}{p}}\Big(\big\|\sum_{i\in\nbb_n}\alpha_ic_j(x_i)\phi_1(x_i)\big\|_2^{\frac{2}{p-1}}\big[\sum_{i\in\nbb_n}\alpha_ic_j(x_i)\phi_1(x_i)\big],\\
  &\qquad\ldots,\big\|\sum_{i\in\nbb_n}\alpha_ic_j(x_i)\phi_M(x_i)\big\|_2^{\frac{2}{p-1}}\big[\sum_{i\in\nbb_n}\alpha_ic_j(x_i)\phi_M(x_i)\big]\Big).
\end{align*}
Note that the above derivation uses Lemma~\ref{lem:gradient-pnorm} from Supplemental Material \ref{app:proofs_dual}.
\end{proof}

%-------------------------------------
\subsection{Proof of Proposition~\ref{prop:mixture-update}}\label{supp:proof_pontil}
%-------------------------------------

\begin{proof}[Proof of Proposition~\ref{prop:mixture-update}]
Fixing the variables $w_j^{(m)}$ and $b$, the optimization problem~\eqref{soft-primal} reduces to
$$
    \min_\beta\;\sum_{j\in\nbb_l}\sum_{m\in\nbb_M}\frac{\|w_j^{(m)}\|_2^2}{2\beta_{jm}}\quad
    \mbox{s.t.}\;\sum_{m\in\nbb_M}\beta^p_{jm}\leq1,\beta_{jm}\geq0,\;\forall j\in\nbb_l,m\in\nbb_M.
$$
This problem can be decomposed into $l$ independent subproblems, one at each locality. For example, the subproblem at the $j$-th locality is as follows
  $$
    \min_\beta\;\sum_{m\in\nbb_M}\frac{\|w_j^{(m)}\|_2^2}{2\beta_{jm}}\quad
    \mbox{s.t.}\;\sum_{m\in\nbb_M}\beta^p_{jm}\leq1,\beta_{jm}\geq0,\;\forall m\in\nbb_M.
  $$
Applying Lemma~\ref{lem:micchelli} with $\alpha_m=\|w_j^{(m)}\|_2^2$, $\eta_m=\beta_{jm}$ and $r=p$ completes the proof.
\end{proof}

%-------------------------------------
\subsection{Proof of Theorem~\ref{thm:clmkl_convergence}: Convergence of the CLMKL Optimization Algorithm}\label{supp:lemm_bcd}
%-------------------------------------

The following lemma is a direct consequence of Lemma 3.1 and Theorem 4.1 in \citet{tseng2001convergence}.

\begin{lemma}\label{lem:convergence-BCD}
  Let $f:\mathbb{R}^{d_1+\cdots+d_R}\to\mathbb{R}\cup\{\infty\}$ be a function. Put $d=d_1+\cdots+d_R$. Suppose that $f$ can be decomposed into $f(\alpha_1,\ldots,\alpha_R)+\sum_{r\in\nbb_R}f_r(\alpha_r)$ for some $f_0:\mathbb{R}^d\to\mathbb{R}\cup\{\infty\}$ and $f_r:\mathbb{R}^{d_r}\to\mathbb{R}\cup\{\infty\},r\in\nbb_R$. Initialize the block coordinate descent method by $\alpha^0=(\alpha^0_1,\ldots,\alpha_R^0)$. Define the iterates $\alpha^k=(\alpha_1^k,\ldots,\alpha_R^k)$ by
  \begin{equation}\label{block-descent-update}
    \alpha_r^{k+1}=\arg\min_{\frak{u}\in\mathbb{R}^{d_k}}f(\alpha_1^{k+1},\ldots,\alpha_{r-1}^{k+1},\frak{u},\alpha^k_{r+1},\ldots,\alpha^k_R),\quad\forall r\in\nbb_R,k\in\nbb_+.
  \end{equation}
  Assume that
  \begin{enumerate}[({A}1)]
    \item $f$ is convex and proper, i.e., $f\not\equiv\infty$
    \item the sublevel set $\mathcal{A}^0:=\{\alpha\in\mathbb{R}^d:f(\alpha)\leq f(\alpha^0)\}$ is compact and $f$ is continuous on $\mathcal{A}^0$
    \item $\text{dom}(f_0):=\{\alpha\in\mathbb{R}^d:f_0(\alpha)\leq \infty\}$ is open and $f_0$ is continuously differentiable on $\text{dom}(f_0)$.
  \end{enumerate}
  Then, the minimizer in \eqref{block-descent-update} exists and any limit point of the sequence $(\alpha^k)_{k\in\nbb_+}$ minimizes $f$ over $\mathcal{A}^0$.
\end{lemma}

The proof of Theorem~\ref{thm:clmkl_convergence} is a direct consequence of the above lemma.

\begin{proof}[Proof of Theorem~\ref{thm:clmkl_convergence}]
The primal problem \eqref{soft-primal} can be rewritten as follows:
\begin{equation}\label{convergence-1}
  \inf_{w,\beta_j\in\Theta_p,j\in\nbb_l}\sum_{j\in\nbb_l}\sum_{m\in\nbb_M}\frac{\|w_j^{(m)}\|_2^2}{2\beta_{jm}}+C\sum_{i\in\nbb_n}\ell(\sum_{j\in\nbb_l}c_j(x_i)\sum_{m\in\nbb_M}\langle w_j^{(m)},\phi_m(x_i)\rangle,y_i),
\end{equation}
where $\Theta_p=\{(\theta_1,\ldots,\theta_M)\in\mathbb{R}^M:\theta_{m}\geq0,\|(\theta_{m})_{m=1}^M\|_p\leq1\}$, and $w_j^{(m)}\in\mathbb{R}^{e_m},\forall j\in\nbb_l,m\in\nbb_M$. Note that Eq. \eqref{convergence-1} can be written as the following unconstrained problem:
$$\inf_{w,\beta}f(w,\beta),\qquad\text{where }f(w,\beta)=f_0(w,\beta)+f_1(w)+f_2(\beta),$$with
\begin{gather*}
  f_0(w,\beta)=\sum_{j\in\nbb_l}\sum_{m\in\nbb_M}\Big[\frac{\|w_j^{(m)}\|_2^2}{2\beta_{jm}}+I_{\beta_{jm}\geq0}(\beta_{jm})\Big]\\
  f_1(w)=C\sum_{i\in\nbb_n}\ell(\sum_{j\in\nbb_l}c_j(x_i)\sum_{m\in\nbb_M}\langle w_j^{(m)},\phi_m(x_i)\rangle,y_i)\quad\text{and }f_2(\beta)=\sum_{j\in\nbb_l}I_{\|\beta_j\|_p\leq1}(\beta_j).
\end{gather*}
Here $I$ is the indicator function, i.e., $I_S(s)=0$ if $s\in S$ and $\infty$ otherwise.

Now, it remains to check the assumption (A1), (A2) and (A3) in Lemma \ref{lem:convergence-BCD} for Algorithm \ref{algorithm:wrapper}.

\textsc{Validity of A1}. It is known that a quadratic over a linear function is convex, so the term $\sum_{j\in\nbb_l}\sum_{m\in\nbb_M}\frac{\|w_j^{(m)}\|_2^2}{2\beta_{jm}}$ is convex. Also, since $\ell$ is convex w.r.t. the first argument and the term $\sum_{j\in\nbb_l}c_j(x_i)\sum_{m\in\nbb_M}\langle w_j^{(m)},\phi_m(x_i)\rangle$ is a linear function of $w$, we immediately know the term $\sum_{i\in\nbb_n}\ell(\sum_{j\in\nbb_l}c_j(x_i)\sum_{m\in\nbb_M}\langle w_j^{(m)},\phi_m(x_i)\rangle,y_i)$ is convex. The convexity of $f$ follows immediately. For the initial assignment $w^{(0)}=\{w_j^{(m,0)}\}_{j,m}$ with $w_{j}^{(m,0)}=0$ and $\beta^{(0)}=\{\beta^{(0)}_{jm}\}_{jm}$ with $\beta^{(0)}_{jm}=M^{-1/p}$, we know
$$f(w^{(0)},\beta^{(0)})=Cn\sup_{y\in\ycal}\ell(0,y)<\infty$$ and therefore $f$ is proper.

\textsc{Validity of A2}. Recall that our algorithm is initialized with $w^{(0)}=0$ and $\beta^{(0)}=M^{-1/p}$. For any $(w,\beta)\in\mathcal{A}^0:=\{(\bar{w},\bar{\beta}):f(\bar{w},\bar{\beta})\leq Cn\sup_{y\in\ycal}\ell(0,y)\}$, we have $$\|w_j^{(m)}\|^2_2\leq 2\beta_{jm}Cn\sup_{y\in\ycal}\ell(0,y)\leq 2Cn\sup_{y\in\ycal}\ell(0,y),$$which, coupled with the constraint $\|(\beta_{jm})_{m=1}^M\|_p\leq1,\forall j\in\nbb_l$, immediately shows that $\mathcal{A}^0$ is bounded. Furthermore, since $f_0,f_1$ and $f_2$ are continuous on their respective domains, the function $f$ is therefore continuous on $\mathcal{A}^0\subset\text{dom}(f_0)\cap\text{dom}(f_1)\cap\text{dom}(f_2)$. It is also known that the preimage of a closed set is closed under a continuous function, from which we know the set $\mathcal{A}^0= f^{-1}(-\infty,f(w^{(0)},\beta^{(0)})]$ is closed. Any closed and bounded subset in $\mathbb{R}^d,d\in\nbb$ is compact and thus $\mathcal{A}^{(0)}$ is compact.

\textsc{Validity of A3}. Clearly, $\text{dom}(f_0)=\{(w,\beta):\beta>0\}$ is open and $f_0$ is continuously differentiable on $\text{dom}(f_0)$.
\end{proof}

%-----------------------------------------------
\subsection{Proof of Generalization Error Bounds (Theorem \ref{thm:empirical-rademacher-bound})}\label{supp:proof_rad}
%-----------------------------------------------

In this section we present the proof of the achieved generalization error bounds (Theorem \ref{thm:genBound} in the main text). Denote $\bar{p}=\frac{2p}{p+1}$ for any $p\geq1$ and observe that $\bar{p}\leq2$, which implies $\bar{p}^*\geq2$. To start with, we give a discussion on the interpretation and tightness of Rademacher complexity bounds in Theorem \ref{thm:empirical-rademacher-bound}.

\paragraph{Interpretation and Tightness of Rademacher complexity bounds}
It can be directly checked that the function $x\to xM^{2/x}$ is decreasing along the interval $(0,2\log M)$ and increasing along the interval $(2\log M,\infty)$.
Therefore, under the assumption $k_m(x,x)\leq B$ the Rademacher complexity bounds thus satisfy the inequalities:
$$
  \hat{R}_n(H_{p,D})\leq\sqrt{\frac{DB}{n}}\times\begin{cases}
    \frac{1}{n}\sqrt{2eDB\log M[\sum_{j\in\nbb_l}\sum_{i\in\nbb_n}c_j^2(x_i)]},&\text{if }p\leq\frac{\log M}{\log M-1},\\
    \frac{1}{n}\sqrt{\frac{2p}{p-1}DBM^{\frac{p-1}{p}}[\sum_{j\in\nbb_l}\sum_{i\in\nbb_n}c_j^2(x_i)]},&\text{otherwise}.
  \end{cases}
$$
In particular, the former expression can be taken for $p=1$, resulting in a mild logarithmic dependence on the number of kernels.
Note that in the limiting case of just one cluster, i.e., $l=1$, the Rademacher complexity bounds match the
result by \citet{cortes2010generalization}, which was shown to be tight.

The proof of Theorem~\ref{thm:empirical-rademacher-bound} is based on the following lemmata.

\begin{lemma}[Khintchine-Kahane inequality~\citep{kahane1985some}]
  Let $v_1,\ldots,v_n\in\hcal$. Then, for any $q\geq1$, it holds
  $$\ebb_{\bm{\sigma}}\big\|\sum_{i\in\nbb_n}\sigma_iv_i\big\|_2^q\leq\bigg(q\sum_{i\in\nbb_n}\|v_i\|_2^2\bigg)^{\frac{q}{2}}.$$
\end{lemma}

\begin{lemma}[Block-structured H\"older inequality~\citep{kloft2012convergence}]
  Let $x=(x^{(1)},\ldots,x^{(n)}),y=(y^{(1)},\ldots,y^{(n)})\in\hcal=\hcal_1\times\cdots\times\hcal_n$. Then, for any $p\geq1$, it holds
  $$\langle x,y\rangle\leq\|x\|_{2,p}\|y\|_{2,p^*}.$$
\end{lemma}

\begin{proof}[Proof of Theorem \ref{thm:empirical-rademacher-bound}]
  Firstly, for any $1\leq\bar{t}\leq2$ we can apply a block-structured version of H\"older inequality to bound $\sum_{i\in\nbb_n}\sigma_if_w(x_i)$ by
  \begin{equation}\label{empirical-rademacher-bound-1}
  \begin{split}
    \sum_{i\in\nbb_n}\sigma_if_w(x_i)&=\sum_{i\in\nbb_n}\sigma_i\Big[\sum_{j\in\nbb_l}c_j(x_i)\sum_{m\in\nbb_M}\langle w_j^{(m)},\phi_m(x_i)\rangle\Big]\\
    &=\sum_{i\in\nbb_n}\sigma_i\Big[\sum_{j\in\nbb_l}c_j(x_i)\langle w_j,\phi(x_i)\rangle\Big]\\
    &=\sum_{j\in\nbb_l}\bigg\langle w_j,\sum_{i\in\nbb_n}\sigma_ic_j(x_i)\phi(x_i)\bigg\rangle\\
    &\stackrel{\text{H\"older}}{\leq}\sum_{j\in\nbb_l}\Big[\|w_j\|_{2,\bar{t}}\Big\|\sum_{i\in\nbb_n}\sigma_ic_j(x_i)\phi(x_i)\Big\|_{2,\bar{t}^*}\Big]\\
    &\stackrel{\text{C.~S}}{\leq}\Big[\sum_{j\in\nbb_l}\|w_j\|^2_{2,\bar{t}}\Big]^{\frac{1}{2}}\Big[\sum_{j\in\nbb_l}\Big\|\sum_{i\in\nbb_n}\sigma_ic_j(x_i)\phi(x_i)\Big\|^2_{2,\bar{t}^*}\Big]^{\frac{1}{2}}.
  \end{split}
  \end{equation}
  For any $j\in\nbb_l$, the Khintchine-Kahane (K.-K.) inequality and Jensen inequality (since $\bar{t}^*\geq2$) permit us to bound $\ebb_{\bm{\sigma}}\Big\|\sum_{i\in \nbb_n}\sigma_ic_j(x_i)\phi(x_i)\Big\|^2_{2,\bar{t}^*}$ by
  \begin{align*}
    \ebb_{\bm{\sigma}}\Big\|\sum_{i\in \nbb_n}\sigma_ic_j(x_i)\phi(x_i)\Big\|^2_{2,\bar{t}^*}&\stackrel{\text{def}}{=}\ebb_{\bm{\sigma}}\bigg[\sum_{m\in\nbb_M}\Big\|\sum_{i\in \nbb_n}\sigma_ic_j(x_i)\phi_m(x_i)\Big\|_2^{\bar{t}^*}\bigg]^{\frac{2}{\bar{t}^*}}\\
    &\stackrel{\text{Jensen}}{\leq}\bigg[\ebb_{\bm{\sigma}}\sum_{m\in\nbb_M}\Big\|\sum_{i\in \nbb_n}\sigma_ic_j(x_i)\phi_m(x_i)\Big\|_2^{\bar{t}^*}\bigg]^{\frac{2}{\bar{t}^*}}\\
    &\stackrel{\text{K.-K.}}{\leq}\bigg[\sum_{m\in\nbb_M}\Big(\bar{t}^*\sum_{i\in \nbb_n}c_j^2(x_i)\|\phi_m(x_i)\|_2^2\Big)^{\frac{\bar{t}^*}{2}}\bigg]^{\frac{2}{\bar{t}^*}}\\
    &= \bar{t}^*\bigg[\sum_{m\in\nbb_M}\Big(\sum_{i\in \nbb_n}c^2_j(x_i)k_m(x_i,x_i)\Big)^{\frac{\bar{t}^*}{2}}\bigg]^{\frac{2}{\bar{t}^*}}\\
    &=\bar{t}^*\bigg\|\Big(\sum_{i\in \nbb_n}c^2_j(x_i)k_m(x_i,x_i)\Big)^M_{m=1}\bigg\|_{\frac{\bar{t}^*}{2}}.
  \end{align*}
  Plugging the above inequalities into Eq.~\eqref{empirical-rademacher-bound-1} and noticing the trivial inequality $\|w_j\|_{2,\bar{t}}\leq\|w_j\|_{2,\bar{p}},\forall t\geq p\geq1$, we get the following bound:
  \begin{multline*}
    \hat{R}_{n}(H_{p,D})\leq\inf_{t\geq p}\hat{R}_n(H_{t,D})
    \leq\frac{\sqrt{D}}{n}\inf_{t\geq p}\bigg(\bar{t}^*\sum_{j\in\nbb_l}\bigg\|\Big(\sum_{i\in \nbb_n}c^2_j(x_i)k_m(x_i,x_i)\Big)^M_{m=1}\bigg\|_{\frac{\bar{t}^*}{2}}\bigg)^{1/2}.
  \end{multline*}
  The above inequality can be equivalently written as Eq. \eqref{empirical-rademacher-bound}.

  Under the condition $k_m(x,x)\leq B$, the term in the brace of Eq. \eqref{empirical-rademacher-bound} can be controlled by
  \begin{align*}
    \sum_{j\in\nbb_l}\bigg\|\Big(\sum_{i\in \nbb_n}c^2_j(x_i)k_m(x_i,x_i)\Big)^M_{m=1}\bigg\|_{\frac{t}{2}}&=\sum_{j\in\nbb_l}\bigg[\sum_{m\in\nbb_M}(\sum_{i\in \nbb_n}c^2_j(x_i)k_m(x_i,x_i))^{\frac{t}{2}}\bigg]^{\frac{2}{t}}\\
    &\leq BM^{\frac{2}{t}}\sum_{j\in\nbb_l}\sum_{i\in\nbb_n}c_j^2(x_i).
  \end{align*}
  Therefore, the inequality \eqref{empirical-rademacher-bound} further translates to
  \begin{align*}
    \hat{R}_{n}(H_{p,D})&\leq\frac{\sqrt{DB}}{n}\inf_{2\leq t\leq\frac{2p}{p-1}}\Big(tM^{\frac{2}{t}}\big[\sum_{j\in\nbb_l}\sum_{i\in\nbb_n}c_j^2(x_i)\big]\Big)^{\frac{1}{2}}.
  \end{align*}
\end{proof}

\begin{proof}[Proof of Theorem \ref{thm:genBound}]
The proof now simply follows by plugging in the bound of Theorem~\ref{thm:empirical-rademacher-bound} into Theorem 7 of \citet{bartlett2002rademacher}.
\end{proof}

%%%%%%%%%%%%%%%%%%%%%%%%%%%%%%%%%%%%%%%%%%%%%%%%%%%%%%%%%%%%%%%%%%%
\section{Support Vector Regression Formulation of CLMKL}\label{sec:regression}
%%%%%%%%%%%%%%%%%%%%%%%%%%%%%%%%%%%%%%%%%%%%%%%%%%%%%%%%%%%%%%%%%%%

For the $\epsilon$-insensitive loss $\ell(t,y)=[|y-t|-\epsilon]_+$, denoting $a_+=\max(a,0)$ for all $a\in\mathbb{R}$,
we have $\ell^*(-\frac{\alpha_i}{C},y_i)=-\frac{1}{C}\alpha_iy_i+\epsilon|\frac{\alpha_i}{C}|$ if $|\alpha_i|\leq C$ and $\infty$ elsewise~\citep{heinrich2012fenchel}.
Hence, the complete dual problem \eqref{soft-complete-dual} reduces to
\begin{equation}\label{soft-complete-dual-insensitive}
\begin{split}
  \sup_{\alpha}&-\frac{1}{2}\sum_{j\in\nbb_l}\Big\|\big(\sum_{i\in\nbb_n}\alpha_ic_j(x_i)\phi_m(x_i)\big)_{m=1}^M\Big\|^2_{2,\frac{2p}{p-1}}+
  \sum_{i\in\nbb_n}(\alpha_iy_i-\epsilon|\alpha_i|),\\
  \mbox{s.t.}&\sum_{i\in\nbb_n}\alpha_i=0,\\
  &|\alpha_i|\leq C,\qquad\forall i\in\nbb_n.
\end{split}
\end{equation}

Let $\alpha_i^+,\alpha_i^-\geq0$ be the positive and negative parts of $\alpha_i$, that is, $\alpha_i=\alpha_i^+-\alpha_i^-,|\alpha_i|=\alpha_i^++\alpha_i^-$. Then, the optimization problem \eqref{soft-complete-dual-insensitive} translates as follows.
\begin{problem}[CLMKL---Regression Problem]
For the $\epsilon$-insensitive loss, the dual CLMKL problem \eqref{soft-complete-dual} is given by:
\begin{equation}
\begin{split}
  \sup_\alpha&-\frac{1}{2}\sum_{j\in\nbb_l}\Big\|\big(\sum_{i\in\nbb_n}c_j(x_i)(\alpha_i^+-\alpha_i^-)\phi_m(x_i)\big)_{m=1}^M\Big\|^2_{2,\frac{2p}{p-1}}+
  \sum_{i\in\nbb_n}(\alpha_i^+-\alpha_i^-)y_i-\epsilon\sum_{i\in\nbb_n}(\alpha_i^++\alpha_i^-)\\
  \mbox{s.t.}&\sum_{i\in\nbb_n}(\alpha_i^+-\alpha_i^-)=0\\
  &0\leq\alpha_i^+,\alpha_i^-\leq C,\alpha_i^+\alpha_i^-=0,\qquad\forall i\in\nbb_n.
\end{split}
\end{equation}
\end{problem}

%%%%%%%%%%%%%%%%%%%%%%%%%%%%%%%%%%%%%%%%%%%%%%%%%%%%%%%%%%%%%%%%%%%
\section{Primal and Dual CLMKL Problem Given Fixed Kernel Weights}\label{supp:partial_dual}
%%%%%%%%%%%%%%%%%%%%%%%%%%%%%%%%%%%%%%%%%%%%%%%%%%%%%%%%%%%%%%%%%%%

Temporarily fixing the kernel weights $\beta$, the partial primal and dual optimization problems become as follows.
\begin{problem}[\sc Primal CLMKL Optimization Problem]
Given a loss function $\ell(t,y):\mathbb{R}\times\mathcal{Y}\to\mathbb{R}$ convex in the first argument and kernel weights $\vbeta=(\beta_{jm})$, solve
\begin{equation}\label{soft-primal-partial}
\begin{split}
  \min_{w,t,b}&\sum_{j\in\nbb_l}\sum_{m\in\nbb_M}\frac{\|w_j^{(m)}\|_2^2}{2\beta_{jm}}+C\sum_{i\in\nbb_n}\ell(t_i,y_i)\\
  \mbox{s.t.}\;\;&\sum_{j\in\nbb_l}\big[c_j(x_i)\sum_{m\in\nbb_M}\langle w_j^{(m)},\phi_m(x_i)\rangle\big]+b=t_i,\;\forall i\in\nbb_n.
\end{split}
\end{equation}
\end{problem}

\begin{problem}[\sc Dual CLMKL Problem---Partially Dualized\label{thm:soft-partial-dual}]
  Given a loss function $\ell(t,y):\mathbb{R}\times\mathcal{Y}\to\mathbb{R}$ convex in the first argument and kernel weights $\vbeta=(\beta_{jm})$, solve
  \begin{equation}\label{soft-partial-dual2}
    \sup_{\valpha:\sum_{i\in\nbb_n}\alpha_i=0}-\frac{1}{2}\sum_{j\in\nbb_l}\sum_{m\in\nbb_M}\beta_{jm}\Big\|\sum_{i\in\nbb_n}\alpha_ic_j(x_i)\phi_m(x_i)\Big\|_2^2-C\sum_{i\in\nbb_n}\ell^*(-\frac{\alpha_i}{C},y_i).
  \end{equation}
  For any feasible dual variables, the primal variable $w_j^{(m)}(\alpha)$ minimizing the associated Lagrangian saddle problem is
  \begin{equation}\label{weight-function-partial-representation2}
    w_j^{(m)}(\alpha)=\beta_{jm}\sum_{i\in\nbb_n}\alpha_ic_j(x_i)\phi_m(x_i).
  \end{equation}
\end{problem}

\begin{derivation}
With the Lagrangian multipliers $\alpha_i,i\in\nbb_n$, the Lagrangian saddle problem of Eq. \eqref{soft-primal-partial} is
\begin{equation}\label{soft-partial-dual-derivation}
\begin{split}
  &\sup_{\alpha}\inf_{w,t,b}\sum_{j\in\nbb_l}\sum_{m\in\nbb_M}\frac{\|w_j^{(m)}\|_2^2}{2\beta_{jm}}+C\sum_{i\in\nbb_n}\ell(t_i,y_i)\\
  &\qquad-\sum_{i\in\nbb_n}\alpha_i\Big(\sum_{j\in\nbb_l}[c_j(x_i)\sum_{m\in\nbb_M}\langle w_j^{(m)},\phi_m(x_i)\rangle]+b-t_i\Big)\\
  &=\sup_{\alpha}\bigg\{-C\sum_{i\in\nbb_n}\sup_{t_i}[-\ell(t_i,y_i)-\frac{1}{C}\alpha_it_i]-\sup_b\sum_{i\in\nbb_n}\alpha_ib-\\
  &\qquad\sum_{j\in\nbb_l}\sum_{m\in\nbb_M}\sup_{w_j^{(m)}}\frac{1}{\beta_{jm}}\Big[\langle w_j^{(m)},\sum_{i\in\nbb_n}\beta_{jm}\alpha_ic_j(x_i)\phi_m(x_i)-\frac{1}{2}\|w_j^{(m)}\|_2^2\Big]\bigg\}\\
  &\stackrel{\text{def}}{=}\sup_{\sum_{i\in\nbb_n}\alpha_i=0}\bigg\{-C\sum_{i\in\nbb_n}\ell^*(-\frac{\alpha_i}{C},y_i)-\frac{1}{2}\sum_{j\in\nbb_l}\sum_{m\in\nbb_M}\beta_{jm}\big\|\sum_{i\in\nbb_n}\alpha_ic_j(x_i)\phi_m(x_i)\big\|_2^2\bigg\}.
\end{split}
\end{equation}
From the above deduction, the variable $w_j^{(m)}(\alpha)$ is a solution of the following problem
$$w_j^{(m)}(\alpha)=\arg\min_{v\in H_m}\big[\langle v,\sum_{i\in\nbb_n}\alpha_i\beta_{jm}c_j(x_i)\phi_m(x_i)\rangle-\frac{1}{2}\|v\|^2_2\big]$$and it can be directly checked that this $w_j^{(m)}(\alpha)$ can be analytically represented by
$$%\begin{equation}\label{weight-function-partial-representation}
  w_j^{(m)}(\alpha)=\beta_{jm}\sum_{i\in\nbb_n}\alpha_ic_j(x_i)\phi_m(x_i).
$$%\end{equation}
\end{derivation}

%Note also that the above derivation also shows that in the optimum we have the following representer theorem:

Plugging the Fenchel conjugate function of the hinge loss and the $\epsilon$-insensitive loss into Problem \ref{thm:soft-partial-dual}, we have the following partial dual problems for the hinge loss and $\epsilon$-insensitive loss. Here $\tilde{k}$ is the kernel defined in Eq. \eqref{kernel-general-loss}.

\begin{problem}[\textsc{Dual CLMKL Problem---Partially Dualized for Hinge Loss}]
\begin{align*}
    \sup_{\alpha_i}\;\;&\sum_{i\in\nbb_n}\alpha_i-\frac{1}{2}\sum_{i,\tilde{i}\in\nbb_n}y_iy_{\tilde{i}}\alpha_i\alpha_{\tilde{i}}\tilde{k}(x_i,x_{\tilde{i}})\\
    \mbox{s.t.}\;\;\;&\sum_{i\in\nbb_n}\alpha_iy_i=0\\
    &0\leq\alpha_i\leq C\qquad\forall i\in\nbb_n.
\end{align*}
\end{problem}

\begin{problem}[\textsc{Dual CLMKL Problem---Partially Dualized for $\epsilon$-insensitive Loss}]
\begin{align*}
    \max_{\alpha_i}\;\;&-\frac{1}{2}\sum_{i,\tilde{i}\in\nbb_n}(\alpha_i^+-\alpha_i^-)(\alpha^+_{\tilde{i}}-\alpha^-_{\tilde{i}})\tilde{k}(x_i,x_{\tilde{i}})+
    \sum_{i\in\nbb_n}(\alpha_i^+-\alpha_i^-)y_i-\epsilon\sum_{i\in\nbb_n}(\alpha_i^++\alpha_i^-)\\
    \mbox{s.t.}&\sum_{i\in\nbb_n}(\alpha_i^+-\alpha_i^-)=0\\
    &0\leq\alpha_i^+,\alpha_i^-\leq C,\alpha_i^+\alpha_i^-=0,\qquad\forall i\in\nbb_n.
\end{align*}
\end{problem}
\section{Details on Our Implementation of Localized MKL}\label{supp:fastGoenen}
%%%%%%%%%%%%%%%%%%%%%%%%%%%%%%%%%%%%%%%%%%%%%%%%%%%%%%%%%%%%%%%%%%%

\citet{gonen2008localized} give the first formulation of localized MKL algorithm by using gating model $\eta_m(x)\propto\exp(\langle v_m,x\rangle+v_{m0})$ to realize locality, and optimize the parameters $v_m,v_{m0},m\in\nbb_M$ with a gradient descent method. However, the calculation of the gradients requires $O(n^2M^2d)$ operations in \citet{gonen2008localized}, which scales poorly w.r.t. the dimension $d$ and the number of kernels. Also, the definition of gating model requires the information of primitive features, which is not accessible in some application areas. For example, data in bioinformatics may appear in a non-vectorial format such as trees and graphs for which the representation of the data with vectors is non-trivial but the calculation of kernel matrices is direct. Although \citet{gonen2013localized} propose to use the empirical feature map $x^{\gcal}=[k_{\gcal}(x_1,x),\ldots,k_{\gcal}(x_n,x)]$ to replace the primitive feature in this case ($k_\gcal$ is a kernel), this turns out to not strictly obey the spirit of the gating function: the empirical feature does not reflect the location of the example in the feature space induced from the kernel. Furthermore, with this strategy the computation of gradient scales as $O(n^3M^2)$, which is quite computationally expensive. In this paper, we give a natural definition of the gating model in a kernel-induced feature space, and provide a fast implementation of the resulting LMKL algorithm. Let $k_0$ be the kernel used to define the gating model, and let $\phi_0$ be the associated feature map. Our basic idea is based on the discovery that the parameter $v_m$ can always be represented as a linear combination of $\phi_0(x_1),\ldots,\phi_0(x_n)$, so the calculation of the representation coefficients is sufficient to restore $v_m$.  We consider the gating model of the form
$$\eta_m(x)=\frac{\exp\big(\langle v_m,\phi_0(x)+v_{m0}\big)}{\sum_{\tilde{m}\in\nbb_M}\exp\big(\langle v_{\tilde{m}},\phi_0(x)+v_{\tilde{m}0}\big)}.$$
\citet{gonen2008localized} proposed to optimize the objective function $$J(v):=\sum_{i\in\nbb_n}\alpha_i-\frac{1}{2}\sum_{i\in\nbb_n}\sum_{\tilde{i}\in\nbb_n}\alpha_i\alpha_{\tilde{i}}y_iy_{\tilde{i}}\big[\sum_{m\in\nbb_M}\eta_m(x_i)k_m(x_i,x_{\tilde{i}})\eta_m(x_{\tilde{i}})\big]$$ with a gradient descent method. The gradient of $J(v)$ can be expressed as:
\begin{equation}\label{gradient-gonen}
  \frac{\partial J(v)}{\partial v_m}=-\sum_{i\in\nbb_n}\sum_{\tilde{i}\in\nbb_n}\sum_{\tilde{m}\in\nbb_M}\alpha_i\alpha_{\tilde{i}}y_iy_{\tilde{i}}\eta_{\tilde{m}}(x_i)k_{\tilde{m}}(x_i,x_{\tilde{i}})\eta_{\tilde{m}}(x_{\tilde{i}})\phi_0(x_i)[\delta_m^{\tilde{m}}-\eta_m(x_i)],
\end{equation}
where $\delta_{m}^{\tilde{m}}=1$ if $m=\tilde{m}$ and $0$ otherwise. Let $v^{(t)}=(v^{(t)}_1,\ldots,v^{(t)}_M)$ be the value of $v=(v_1,\ldots,v_M)$ at the $t$-th iteration and let $r^{(t)}(i,m)$ be the representation of $v_m^{(t)}$ in terms of $\phi_0(x_i)$, i.e., $v^{(t)}_m=\sum_{i\in\nbb_n}r^{(t)}(i,m)\phi_0(x_i)$. Analogously, let $g^{(t)}(i,m)$ be the representation coefficient of $\partial J(v^{(t)})/\partial v_m$ in terms of $\phi_0(x_i)$. Introduce two arrays for convenience:
$$
  B(i,m)=\sum_{\tilde{i}\in\nbb_n}\alpha_{\tilde{i}}y_{\tilde{i}}k_{m}(x_i,x_{\tilde{i}})\eta_{m}(x_{\tilde{i}}),\quad
  A(i)=\sum_{m\in\nbb_M}\eta_{m}(x_i)B(i,m),\quad i\in\nbb_n,m\in\nbb_M.
$$%\end{gather*}%\sum_{\tilde{m}\in\nbb_M}\eta_{\tilde{m}}(x_i)\sum_{\tilde{i}\in\nbb_n}\alpha_{\tilde{i}}y_{\tilde{i}}\eta_{\tilde{m}}(x_{\tilde{i}})k_{\tilde{m}}(x_i,x_{\tilde{i}})
Eq. \eqref{gradient-gonen} then implies that
\begin{equation}\label{gradient-gonen-calculation}
\begin{split}
  g^{(t)}(i,m)&=-\sum_{\tilde{i}\in\nbb_n}\sum_{\tilde{m}\in\nbb_M}\alpha_i\alpha_{\tilde{i}}y_iy_{\tilde{i}}\eta_{\tilde{m}}(x_i)k_{\tilde{m}}(x_i,x_{\tilde{i}})\eta_{\tilde{m}}(x_{\tilde{i}})[\delta_m^{\tilde{m}}-\eta_m(x_i)]\\
%  &=-\sum_{\tilde{i}\in\nbb_n}\alpha_i\alpha_{\tilde{i}}y_iy_{\tilde{i}}\eta_m(x_i)k_m(x_i,x_{\tilde{i}})\eta_m(x_{\tilde{i}})+\sum_{\tilde{i}\in\nbb_n}\sum_{\tilde{m}\in\nbb_M}\alpha_i\alpha_{\tilde{i}}y_iy_{\tilde{i}}\eta_{\tilde{m}}(x_i)k_{\tilde{m}}(x_i,x_{\tilde{i}})\eta_{\tilde{m}}(x_{\tilde{i}})\eta_m(x_i)\\
  &=-\alpha_iy_i\eta_m(x_i)\Big[\sum_{\tilde{i}\in\nbb_n}\alpha_{\tilde{i}}y_{\tilde{i}}k_m(x_i,x_{\tilde{i}})\eta_m(x_{\tilde{i}})-\sum_{\tilde{i}\in\nbb_n}\sum_{\tilde{m}\in\nbb_M}\alpha_{\tilde{i}}y_{\tilde{i}}\eta_{\tilde{m}}(x_i)k_{\tilde{m}}(x_i,x_{\tilde{i}})\eta_{\tilde{m}}(x_{\tilde{i}})\Big]\\
  &=-\alpha_iy_i\eta_m(x_i)[B(i,m)-A(i)].
\end{split}
\end{equation}
With the line search $v_m^{(t+1)}=v^{(t)}_m+\mu^{(t)}\frac{\partial J(v^{(t)})}{\partial v_m},m\in\nbb_M$, the representation coefficient can be simply updated by taking $$r^{(t+1)}(i,m)=r^{(t)}(i,m)-\mu^{(t)}g^{(t)}(i,m),\quad\forall i\in\nbb_n,m\in\nbb_M.$$Also, in the calculation of gating model, we need to calculate $\langle \phi_0(x_i),v_m^{(t)}\rangle$, and this can be fulfilled by
$$\langle \phi_0(x_i),v_m^{(t)}\rangle=\sum_{\tilde{i}\in\nbb_n}\langle\phi_0(x_i),r^{(t)}(\tilde{i},m)\phi_0(x_{\tilde{i}})\rangle=\sum_{\tilde{i}\in\nbb_n}k_0(x_i,x_{\tilde{i}})r^{(t)}(\tilde{i},m).$$ At each iteration, we can use $O(n^2M)$ operations to calculate the arrays $A, B$. Subsequently, the calculation of the gradients as illustrated by Eq. \eqref{gradient-gonen-calculation} can be fulfilled with $O(nM)$ operations. The updating of the representation coefficients $r^{(t)}(i,m)$ requires $O(nM)$ operations, while calculating the gating model $\eta_m(x_i)$ requires further $O(n^2M)$ operations. Putting the above discussions together, our implementation of LMKL based on the kernel trick requires $O(n^2M)$ operations at each iteration, which is much faster than the original implementation in \citet{gonen2008localized} with $O(n^2M^2d)$ operations at each iteration. Here, $d$ is the dimension of the primitive feature.

%%%%%%%%%%%%%%%%%%%%%%%%%%%%%%%%%%%%%%%%%%%%%%%%%%%%%%%%%%%%%%%%%%%
\section{Background on the Experimental Setup and Empirical Results}\label{supp:exp_details}
\subsection{Details on the Protein Fold Prediction Experiment}\label{supp_protein}
%-------------------------------------

We precisely replicate the experimental setup of previous experiments by \citet{campbell2011learning,kloft2011lpb,KloBla11}, so
we use the train/test split supplied by \citet{campbell2011learning} and perform CLMKL via one-versus-all strategy to tackle multiple classes.
We apply kernel k-means to the uniform kernel to generate a partition with $3$ clusters for CLMKL and HLMKL, and, since we have no access to primitive features, use this kernel to define gating model in LMKL.
All the base kernel matrices are multiplicatively normalized before training.
We validate the regularization parameter $C$ over $10^{\{-1,-0.5,\ldots,2\}}$, and the average evenesses over the interval $[0.4,0.7]$ with eight linearly equally spaced points.
Note that the model parameters are tuned separately for each training set and only based on the training set, not the test set.
We repeat the experiment $15$ times.

%-------------------------------------
\subsection{Details on the Visual Image Categorization Experiment}\label{supp_uiuc}
%-------------------------------------

We compute 9 bag-of-words features, each with a dictionary size of 512, resulting in 9 $\chi^2$-Kernels \citep{DBLP:journals/ijcv/ZhangMLS07}.
The first 6 bag-of-words features are computed over SIFT features \citep{DBLP:journals/ijcv/Lowe04} at three different scales and the two color channel sets RGB and opponent colors \citep{DBLP:journals/pami/SandeGS10}.
The remaining 3 bag-of-words features are computed over quantiles of color values at the same three scales.
The quantiles are concatenated over RGB channels.
For each channel within a set of color channels, the quantiles are concatenated.
Local features are extracted at a grid of step size 5 on images that were down-scaled to 600 pixels in the largest dimension.
Assignment of local features to visual words is done using rank-mapping \citep{binder2013enhanced}.
The kernel width of the kernels is set to the mean of the $\chi^2$-distances.
All kernels are multiplicatively normalized.

The dataset is split into 11 parts for outer crossvalidation.
The performance reported in Table \ref{tab:result-uiuc} is the average over the 11 test splits of the outer cross validation.
For each outer cross validation training split, a 10-fold inner crossvalidation is performed for determining optimal parameters.
The parameters are selected using only the samples of the outer training split.
This avoids to report a result merely on the most favorable train test split from the outer cross validation.
For the proposed CLMKL we employ kernel k-means with 3 clusters on the outer training split of the dataset.

We compare CLMKL to regular $\ell_p$-norm MKL \citep{kloft2011lp} and to localized MKL as in \cite{gonen2008localized}.
For all methods, we employ a one-versus-all setup, running over $\ell_p$-norms in $\{1.125,1.333, 2\}$ and regularization constants in $\{10^{k/2}\}_{k=0}^{k=5}$ (optima attained inside the respective grids).
CLMKL uses the same set of $\ell_p$-norms, regularization constants from $\{10^{k/2}\}_{k=0,\ldots,5}$, and average excesses in $\{0.5+i/12\}_{i=0}^{i=5}$.
Performance is measured through multi-class classification accuracy.

\begin{small}
\setlength{\bibsep}{0.03cm}
%\bibliographystyle{ieeetr}
%\bibliography{lei63}

\begin{thebibliography}{57}
\providecommand{\natexlab}[1]{#1}
\providecommand{\url}[1]{\texttt{#1}}
\expandafter\ifx\csname urlstyle\endcsname\relax
  \providecommand{\doi}[1]{doi: #1}\else
  \providecommand{\doi}{doi: \begingroup \urlstyle{rm}\Url}\fi

\bibitem[Abeel et~al.(2009)Abeel, Van~de Peer, and Saeys]{AbeelPS09}
T.~Abeel, Y.~Van~de Peer, and Y.~Saeys.
\newblock Toward a gold standard for promoter prediction evaluation.
\newblock \emph{Bioinformatics}, 25\penalty0 (12):\penalty0 i313--i320, 2009.

\bibitem[Bach et~al.(2004)Bach, Lanckriet, and Jordan]{bach2004multiple}
F.~R. Bach, G.~R. Lanckriet, and M.~I. Jordan.
\newblock Multiple kernel learning, conic duality, and the smo algorithm.
\newblock In \emph{ICML}, page~6, 2004.

\bibitem[Bartlett and Mendelson(2002)]{bartlett2002rademacher}
P.~Bartlett and S.~Mendelson.
\newblock Rademacher and {G}aussian complexities: Risk bounds and structural
  results.
\newblock \emph{Journal of Machine Learning Research}, 3:\penalty0 463--482,
  2002.

\bibitem[Ben-Hur et~al.(2008)Ben-Hur, Ong, Sonnenburg, Sch\"olkopf, and
  R\"atsch]{BenOngSonSchRae08}
A.~Ben-Hur, C.~S. Ong, S.~Sonnenburg, B.~Sch\"olkopf, and G.~R\"atsch.
\newblock Support vector machines and kernels for computational biology.
\newblock \emph{PLoS Computational Biology}, 4, 2008.

\bibitem[Binder et~al.(2013)Binder, Samek, M{\"u}ller, and
  Kawanabe]{binder2013enhanced}
A.~Binder, W.~Samek, K.-R. M{\"u}ller, and M.~Kawanabe.
\newblock Enhanced representation and multi-task learning for image annotation.
\newblock \emph{Computer Vision and Image Understanding}, 2013.

\bibitem[Boyd and Vandenberghe(2004)]{boyd2004convex}
S.~P. Boyd and L.~Vandenberghe.
\newblock \emph{Convex optimization}.
\newblock Cambridge Univ. Press, New York, 2004.

\bibitem[Campbell and Ying(2011)]{campbell2011learning}
C.~Campbell and Y.~Ying.
\newblock Learning with support vector machines.
\newblock \emph{Synthesis Lectures on Artificial Intelligence and Machine
  Learning}, 5\penalty0 (1):\penalty0 1--95, 2011.

\bibitem[Chang and Lin(2011)]{chang2011libsvm}
C.-C. Chang and C.-J. Lin.
\newblock {LIBSVM}: a library for support vector machines.
\newblock \emph{ACM Transactions on Intelligent Systems and Technology (TIST)},
  2\penalty0 (3):\penalty0 27, 2011.

\bibitem[Cortes(2009)]{Cor09}
C.~Cortes.
\newblock Invited talk: Can learning kernels help performance?
\newblock In \emph{Proceedings of the 26th Annual International Conference on
  Machine Learning}, ICML '09, pages 1:1--1:1, New York, NY, USA, 2009. ACM.

\bibitem[Cortes et~al.(2010)Cortes, Mohri, and
  Rostamizadeh]{cortes2010generalization}
C.~Cortes, M.~Mohri, and A.~Rostamizadeh.
\newblock Generalization bounds for learning kernels.
\newblock In \emph{Proceedings of the 28th International Conference on Machine
  Learning}, ICML'10, 2010.

\bibitem[Cortes et~al.(2013)Cortes, Kloft, and Mohri]{cortes2013learning}
C.~Cortes, M.~Kloft, and M.~Mohri.
\newblock Learning kernels using local rademacher complexity.
\newblock In \emph{Advances in Neural Information Processing Systems}, pages
  2760--2768, 2013.

\bibitem[Dhillon et~al.(2004)Dhillon, Guan, and Kulis]{dhillon2004kernel}
I.~S. Dhillon, Y.~Guan, and B.~Kulis.
\newblock Kernel k-means: spectral clustering and normalized cuts.
\newblock In \emph{ACM SIGKDD international conference on Knowledge discovery
  and data mining}, pages 551--556. ACM, 2004.

\bibitem[Ding and Dubchak(2001)]{ding2001multi}
C.~H. Ding and I.~Dubchak.
\newblock Multi-class protein fold recognition using support vector machines
  and neural networks.
\newblock \emph{Bioinformatics}, 17\penalty0 (4):\penalty0 349--358, 2001.

\bibitem[G{\"o}nen and Alpaydin(2008)]{gonen2008localized}
M.~G{\"o}nen and E.~Alpaydin.
\newblock Localized multiple kernel learning.
\newblock In \emph{Proceedings of the 25th international conference on Machine
  learning}, pages 352--359. ACM, 2008.

\bibitem[G\"{o}nen and Alpaydin(2011)]{Goenen11}
M.~G\"{o}nen and E.~Alpaydin.
\newblock Multiple kernel learning algorithms.
\newblock \emph{J. Mach. Learn. Res.}, 12:\penalty0 2211--2268, July 2011.
\newblock ISSN 1532-4435.

\bibitem[G{\"o}nen and Alpayd{\i}n(2013)]{gonen2013localized}
M.~G{\"o}nen and E.~Alpayd{\i}n.
\newblock Localized algorithms for multiple kernel learning.
\newblock \emph{Pattern Recognition}, 46\penalty0 (3):\penalty0 795--807, 2013.

\bibitem[G{\"o}rnitz et~al.(2009)G{\"o}rnitz, Kloft, and
  Brefeld]{gornitz2009active}
N.~G{\"o}rnitz, M.~Kloft, and U.~Brefeld.
\newblock Active and semi-supervised data domain description.
\newblock In \emph{Joint European Conference on Machine Learning and Knowledge
  Discovery in Databases}, pages 407--422. Springer, 2009.

\bibitem[G{\"o}rnitz et~al.(2013)G{\"o}rnitz, Kloft, Rieck, and
  Brefeld]{gornitz2013toward}
N.~G{\"o}rnitz, M.~M. Kloft, K.~Rieck, and U.~Brefeld.
\newblock Toward supervised anomaly detection.
\newblock \emph{Journal of Artificial Intelligence Research}, 2013.

\bibitem[Han and Liu(2012)]{han2012probability}
Y.~Han and G.~Liu.
\newblock Probability-confidence-kernel-based localized multiple kernel
  learning with norm.
\newblock \emph{IEEE Transactions on Systems, Man, and Cybernetics, Part B},
  42\penalty0 (3):\penalty0 827--837, 2012.

\bibitem[Heinrich(2012)]{heinrich2012fenchel}
A.~Heinrich.
\newblock \emph{Fenchel duality-based algorithms for convex optimization
  problems with applications in machine learning and image restoration}.
\newblock PhD thesis, Chemnitz University of Technology, 2012.

\bibitem[Hocking et~al.(2011)Hocking, Joulin, Bach, Vert,
  et~al.]{hocking2011clusterpath}
T.~D. Hocking, A.~Joulin, F.~Bach, J.-P. Vert, et~al.
\newblock Clusterpath: an algorithm for clustering using convex fusion
  penalties.
\newblock In \emph{28th international conference on machine learning}, 2011.

\bibitem[Hussain and Shawe-Taylor(2011)]{hussain2011improved}
Z.~Hussain and J.~Shawe-Taylor.
\newblock Improved loss bounds for multiple kernel learning.
\newblock In \emph{AISTATS}, pages 370--377, 2011.

\bibitem[Kahane(1985)]{kahane1985some}
J.-P. Kahane.
\newblock \emph{Some random series of functions}.
\newblock Cambridge University Press, Cambridge, 1985.

\bibitem[Kloft(2011)]{kloft2011lpb}
M.~Kloft.
\newblock \emph{$\ell_p$-norm multiple kernel learning}.
\newblock PhD thesis, Berlin Institute of Technology, 2011.

\bibitem[Kloft and Blanchard(2011)]{KloBla11}
M.~Kloft and G.~Blanchard.
\newblock The local {R}ademacher complexity of $\ell_p$-norm multiple kernel
  learning.
\newblock In \emph{Advances in Neural Information Processing Systems 24}, pages
  2438--2446. MIT Press, 2011.

\bibitem[Kloft and Blanchard(2012)]{kloft2012convergence}
M.~Kloft and G.~Blanchard.
\newblock On the convergence rate of lp-norm multiple kernel learning.
\newblock \emph{Journal of Machine Learning Research}, 13\penalty0
  (1):\penalty0 2465--2502, 2012.

\bibitem[Kloft et~al.(2008{\natexlab{a}})Kloft, Brefeld, D{\"u}essel, Gehl, and
  Laskov]{kloft2008automatic}
M.~Kloft, U.~Brefeld, P.~D{\"u}essel, C.~Gehl, and P.~Laskov.
\newblock Automatic feature selection for anomaly detection.
\newblock In \emph{Proceedings of the 1st ACM workshop on Workshop on AISec},
  pages 71--76. ACM, 2008{\natexlab{a}}.

\bibitem[Kloft et~al.(2008{\natexlab{b}})Kloft, Brefeld, Laskov, and
  Sonnenburg]{kloft2008non}
M.~Kloft, U.~Brefeld, P.~Laskov, and S.~Sonnenburg.
\newblock Non-sparse multiple kernel learning.
\newblock In \emph{NIPS Workshop on Kernel Learning: Automatic Selection of
  Optimal Kernels}, volume~4, 2008{\natexlab{b}}.

\bibitem[Kloft et~al.(2009)Kloft, Brefeld, Laskov, M{\"u}ller, Zien, and
  Sonnenburg]{kloft2009efficient}
M.~Kloft, U.~Brefeld, P.~Laskov, K.-R. M{\"u}ller, A.~Zien, and S.~Sonnenburg.
\newblock Efficient and accurate lp-norm multiple kernel learning.
\newblock In \emph{Advances in neural information processing systems}, pages
  997--1005, 2009.

\bibitem[Kloft et~al.(2010)Kloft, R{\"u}ckert, and Bartlett]{kloft2010unifying}
M.~Kloft, U.~R{\"u}ckert, and P.~Bartlett.
\newblock A unifying view of multiple kernel learning.
\newblock \emph{Machine Learning and Knowledge Discovery in Databases}, pages
  66--81, 2010.

\bibitem[Kloft et~al.(2011)Kloft, Brefeld, Sonnenburg, and Zien]{kloft2011lp}
M.~Kloft, U.~Brefeld, S.~Sonnenburg, and A.~Zien.
\newblock Lp-norm multiple kernel learning.
\newblock \emph{The Journal of Machine Learning Research}, 12:\penalty0
  953--997, 2011.

\bibitem[Lanckriet et~al.(2004)Lanckriet, Cristianini, Bartlett, Ghaoui, and
  Jordan]{lanckriet2004learning}
G.~R. Lanckriet, N.~Cristianini, P.~Bartlett, L.~E. Ghaoui, and M.~I. Jordan.
\newblock Learning the kernel matrix with semidefinite programming.
\newblock \emph{The Journal of Machine Learning Research}, 5:\penalty0 27--72,
  2004.

\bibitem[Lei and Ding(2014)]{lei2014refined}
Y.~Lei and L.~Ding.
\newblock Refined {R}ademacher chaos complexity bounds with applications to the
  multikernel learning problem.
\newblock \emph{Neural. Comput.}, 26\penalty0 (4):\penalty0 739--760, 2014.

\bibitem[Lei et~al.(2015)Lei, Binder, Dogan, and Kloft]{lei2015theory}
Y.~Lei, A.~Binder, {\"U}.~Dogan, and M.~Kloft.
\newblock Theory and algorithms for the localized setting of learning kernels.
\newblock In \emph{Proceedings of The 1st International Workshop on “Feature
  Extraction: Modern Questions and Challenges”, NIPS}, pages 173--195, 2015.

\bibitem[Li and Fei-Fei(2007)]{li2007andli}
L.-J. Li and L.~Fei-Fei.
\newblock What, where and who? classifying events by scene and object
  recognition.
\newblock In \emph{Computer Vision, 2007. ICCV 2007. IEEE 11th International
  Conference on}, pages 1--8. IEEE, 2007.

\bibitem[Li et~al.(2016)Li, Liu, Wang, Dou, Yin, and Zhu]{limultiple}
M.~Li, X.~Liu, L.~Wang, Y.~Dou, J.~Yin, and E.~Zhu.
\newblock Multiple kernel clustering with local kernel alignment maximization.
\newblock In \emph{Proceedings of the Twenty-Fifth International Joint
  Conference on Artificial Intelligence}, IJCAI'16, 2016.

\bibitem[Liu et~al.(2014)Liu, Wang, Zhang, and Yin]{LiuWZY14}
X.~Liu, L.~Wang, J.~Zhang, and J.~Yin.
\newblock Sample-adaptive multiple kernel learning.
\newblock In \emph{Proceedings of the Twenty-Eighth {AAAI} Conference on
  Artificial Intelligence, July 27 -31, 2014, Qu{\'{e}}bec City, Qu{\'{e}}bec,
  Canada.}, pages 1975--1981, 2014.

\bibitem[Liu et~al.(2015)Liu, Wang, Yin, Dou, and Zhang]{LiuWYDZ15}
X.~Liu, L.~Wang, J.~Yin, Y.~Dou, and J.~Zhang.
\newblock Absent multiple kernel learning.
\newblock In \emph{Proceedings of the Twenty-Ninth {AAAI} Conference on
  Artificial Intelligence, January 25-30, 2015, Austin, Texas, {USA.}}, pages
  2807--2813, 2015.

\bibitem[Lowe(2004)]{DBLP:journals/ijcv/Lowe04}
D.~G. Lowe.
\newblock Distinctive image features from scale-invariant keypoints.
\newblock \emph{International Journal of Computer Vision}, 60\penalty0
  (2):\penalty0 91--110, 2004.

\bibitem[Micchelli and Pontil(2005)]{micchelli2005learning}
C.~A. Micchelli and M.~Pontil.
\newblock Learning the kernel function via regularization.
\newblock \emph{Journal of Machine Learning Research}, pages 1099--1125, 2005.

\bibitem[Moeller et~al.(2016)Moeller, Swaminathan, and
  Venkatasubramanian]{moeller2016unified}
J.~Moeller, S.~Swaminathan, and S.~Venkatasubramanian.
\newblock A unified view of localized kernel learning.
\newblock \emph{arXiv preprint arXiv:1603.01374}, 2016.

\bibitem[Mu and Zhou(2011)]{mu2011non}
Y.~Mu and B.~Zhou.
\newblock Non-uniform multiple kernel learning with cluster-based gating
  functions.
\newblock \emph{Neurocomputing}, 74\penalty0 (7):\penalty0 1095--1101, 2011.

\bibitem[Rakotomamonjy et~al.(2008)Rakotomamonjy, Bach, Canu, Grandvalet,
  et~al.]{rakotomamonjy2008simplemkl}
A.~Rakotomamonjy, F.~Bach, S.~Canu, Y.~Grandvalet, et~al.
\newblock Simplemkl.
\newblock \emph{Journal of Machine Learning Research}, 9:\penalty0 2491--2521,
  2008.

\bibitem[Sch{\"o}lkopf and Smola(2002)]{SchSmo02}
B.~Sch{\"o}lkopf and A.~Smola.
\newblock \emph{Learning with Kernels}.
\newblock {MIT} Press, Cambridge, MA, 2002.

\bibitem[Sonnenburg et~al.(2006{\natexlab{a}})Sonnenburg, R{\"a}tsch,
  Sch{\"a}fer, and Sch{\"o}lkopf]{sonnenburg2006large}
S.~Sonnenburg, G.~R{\"a}tsch, C.~Sch{\"a}fer, and B.~Sch{\"o}lkopf.
\newblock Large scale multiple kernel learning.
\newblock \emph{The Journal of Machine Learning Research}, 7:\penalty0
  1531--1565, 2006{\natexlab{a}}.

\bibitem[Sonnenburg et~al.(2006{\natexlab{b}})Sonnenburg, Zien, and
  R{\"a}tsch]{SonZieRae06}
S.~Sonnenburg, A.~Zien, and G.~R{\"a}tsch.
\newblock Arts: accurate recognition of transcription starts in human.
\newblock \emph{Bioinformatics}, 22\penalty0 (14):\penalty0 e472--e480,
  2006{\natexlab{b}}.

\bibitem[Sonnenburg et~al.(2008)Sonnenburg, Zien, Philips, and
  R{\"a}tsch]{sonnenburg2008poims}
S.~Sonnenburg, A.~Zien, P.~Philips, and G.~R{\"a}tsch.
\newblock Poims: positional oligomer importance matrices—understanding
  support vector machine-based signal detectors.
\newblock \emph{Bioinformatics}, 24\penalty0 (13):\penalty0 i6--i14, 2008.

\bibitem[Srebro and Ben-David(2006)]{srebro2006learning}
N.~Srebro and S.~Ben-David.
\newblock Learning bounds for support vector machines with learned kernels.
\newblock In \emph{COLT}, pages 169--183. Springer-Verlag, Berlin, 2006.

\bibitem[Sun et~al.(2010)Sun, Ampornpunt, Varma, and
  Vishwanathan]{sun2010multiple}
Z.~Sun, N.~Ampornpunt, M.~Varma, and S.~Vishwanathan.
\newblock Multiple kernel learning and the smo algorithm.
\newblock In \emph{Advances in neural information processing systems}, pages
  2361--2369, 2010.

\bibitem[Tseng(2001)]{tseng2001convergence}
P.~Tseng.
\newblock Convergence of a block coordinate descent method for
  nondifferentiable minimization.
\newblock \emph{Journal of optimization theory and applications}, 109\penalty0
  (3):\penalty0 475--494, 2001.

\bibitem[van~de Sande et~al.(2010)van~de Sande, Gevers, and
  Snoek]{DBLP:journals/pami/SandeGS10}
K.~E.~A. van~de Sande, T.~Gevers, and C.~G.~M. Snoek.
\newblock Evaluating color descriptors for object and scene recognition.
\newblock \emph{IEEE Trans. Pattern Anal. Mach. Intell.}, 32\penalty0
  (9):\penalty0 1582--1596, 2010.

\bibitem[Vogt et~al.(2015)Vogt, Kloft, Stark, Raman, Prabhakaran, Roth, and
  R{\"a}tsch]{vogt2015probabilistic}
J.~E. Vogt, M.~Kloft, S.~Stark, S.~S. Raman, S.~Prabhakaran, V.~Roth, and
  G.~R{\"a}tsch.
\newblock Probabilistic clustering of time-evolving distance data.
\newblock \emph{Machine Learning}, 100\penalty0 (2-3):\penalty0 635--654, 2015.

\bibitem[Xu et~al.(2010)Xu, Jin, Yang, King, and Lyu]{xu2010simple}
Z.~Xu, R.~Jin, H.~Yang, I.~King, and M.~R. Lyu.
\newblock Simple and efficient multiple kernel learning by group lasso.
\newblock In \emph{Proceedings of the 27th international conference on machine
  learning (ICML-10)}, pages 1175--1182, 2010.

\bibitem[Yang et~al.(2011)Yang, Xu, Ye, King, and Lyu]{yang2011efficient}
H.~Yang, Z.~Xu, J.~Ye, I.~King, and M.~R. Lyu.
\newblock Efficient sparse generalized multiple kernel learning.
\newblock \emph{IEEE Transactions on Neural Networks}, 22\penalty0
  (3):\penalty0 433--446, 2011.

\bibitem[Yang et~al.(2009)Yang, Li, Tian, Duan, and Gao]{yang2009group}
J.~Yang, Y.~Li, Y.~Tian, L.~Duan, and W.~Gao.
\newblock Group-sensitive multiple kernel learning for object categorization.
\newblock In \emph{2009 IEEE 12th International Conference on Computer Vision},
  pages 436--443. IEEE, 2009.

\bibitem[Ying and Campbell(2009)]{ying2009generalization}
Y.~Ying and C.~Campbell.
\newblock Generalization bounds for learning the kernel.
\newblock In \emph{COLT}, 2009.

\bibitem[Zhang et~al.(2007)Zhang, Marszalek, Lazebnik, and
  Schmid]{DBLP:journals/ijcv/ZhangMLS07}
J.~Zhang, M.~Marszalek, S.~Lazebnik, and C.~Schmid.
\newblock Local features and kernels for classification of texture and object
  categories: A comprehensive study.
\newblock \emph{International Journal of Computer Vision}, 73\penalty0
  (2):\penalty0 213--238, 2007.

\end{thebibliography}

\end{small}

\end{document}